\documentclass[runningheads]{llncs}
\usepackage{graphicx}
\usepackage{amsmath,amsfonts,amssymb}
\usepackage{caption}
\usepackage{subcaption}
\usepackage{bm}
\usepackage{xcolor}
\usepackage{enumitem}

\usepackage[colorinlistoftodos]{todonotes}

\allowdisplaybreaks
\usepackage[sort, numbers]{natbib}
\usepackage{xspace}
\usepackage{algpseudocode}
\usepackage{algorithm}
\usepackage{xurl}

\DeclareRobustCommand{\eg}{e.g.,\@\xspace}                 
\DeclareRobustCommand{\ie}{i.e.,\@\xspace}                    
\DeclareRobustCommand{\wrt}{w.r.t.\@\xspace}

\DeclareMathOperator{\E}{\mathbb{E}}
\DeclareRobustCommand{\algnameMain}{Non-Linear Correlated Target-Feature Aggregation\@\xspace}     
\DeclareRobustCommand{\algnameshortMain}{NonLinCTFA\@\xspace} 
\begin{document}
\algnewcommand\algorithmicforeach{\textbf{for each}}
\algdef{S}[FOR]{ForEach}[1]{\algorithmicforeach\ #1\ \algorithmicdo}
\renewcommand{\algorithmicrequire}{\textbf{Input:}}
\renewcommand{\algorithmicensure}{\textbf{Output:}}

\title{Interpetable Target-Feature Aggregation for Multi-Task Learning based on Bias-Variance Analysis}
\titlerunning{Interpetable Target-Feature Aggregation}
%
\author{Paolo Bonetti\inst{1} \and
Alberto Maria Metelli\inst{1} \and
Marcello Restelli\inst{1}}
\authorrunning{P. Bonetti et al.}
%
\institute{Diparimento di Elettronica, Informazione e Bioingegneria, Politecnico di Milano, Milan, Italy \email{\{paolo.bonetti,albertomaria.metelli,marcello.restelli\}@polimi.it}
}
\maketitle              
\begin{abstract}
Multi-task learning (MTL) is a powerful machine learning paradigm designed to leverage shared knowledge across tasks to improve generalization and performance. Previous works have proposed approaches to MTL that can be divided into feature learning, focused on the identification of a common feature representation, and task clustering, where similar tasks are grouped together. In this paper, we propose an MTL approach at the intersection between task clustering and feature transformation based on a two-phase iterative aggregation of targets and features. First, we propose a bias-variance analysis for regression models with additive Gaussian noise, where we provide a general expression of the asymptotic bias and variance of a task, considering a linear regression trained on aggregated input features and an aggregated target. Then, we exploit this analysis to provide a two-phase MTL algorithm (\algnameshortMain). Firstly, this method partitions the tasks into clusters and aggregates each obtained group of targets with their mean. Then, for each aggregated task, it aggregates subsets of features with their mean in a dimensionality reduction fashion. In both phases, a key aspect is to preserve the interpretability of the reduced targets and features through the aggregation with the mean, which is further motivated by applications to Earth science. Finally, we validate the algorithms on synthetic data, showing the effect of different parameters and real-world datasets, exploring the validity of the proposed methodology on classical datasets, recent baselines, and Earth science applications.

\end{abstract}

\section{Introduction}\label{sec:intro}

\emph{Machine Learning}~\citep[ML]{bishop2006} approaches usually consider an individual learning problem, eventually decomposing complex problems into independent tasks. Inspired by the possibility of exploiting their interconnections,
\emph{Multi-task learning}~\citep[MTL]{caruana1997multitask} methods are designed to simultaneously learn multiple related tasks, leveraging shared knowledge to improve performance and generalization. In recent years, MTL has gained significant attention across various domains, including natural language processing~\citep{chen2019bert}, computer vision~\citep{bachmann2022multimae}, and healthcare~\citep{zhao2023multi}.

Focusing on supervised MTL, \emph{task-clustering} methods are designed to aggregate different tasks into clusters, exploiting their relationships to learn groups of tasks with the same model. On the other hand, \emph{feature-based} MTL approaches are focused on the identification of a subset of relevant common input features (feature selection), or the extraction of a combination of relevant original inputs (feature transformation). In the feature transformation context, a \emph{dimensionality reduction} approach may be considered to extract a set of relevant features common to all the tasks, reducing the dimension of the feature space.

In this paper, we propose an MTL approach (\algnameshortMain) at the intersection between task clustering and feature transformation with dimensionality reduction. Firstly, we introduce a specific task clustering, based on partitioning the targets, aggregating each obtained group of targets with their mean. Then, for each aggregated task (the mean of the targets of a cluster), we aggregate subsets of features with their mean. This way, we first learn a single model for each group of tasks. Then, we aggregate subsets of features with their mean in each aggregated task, producing a set of reduced features in a dimensionality reduction fashion. In both cases, we provide theoretical guarantees on the improvement (or not worsening) of the mean squared error on each of the original tasks. A schematic view of this methodology can be found in Figure \ref{fig:workflowMain}. Additionally, Figure \ref{fig:workflowVar} shows a variant that will be discussed in the next sections.

\begin{figure}[t]
     \centering
     \begin{subfigure}{0.98\textwidth}
         \centering
     \includegraphics[width=\textwidth]{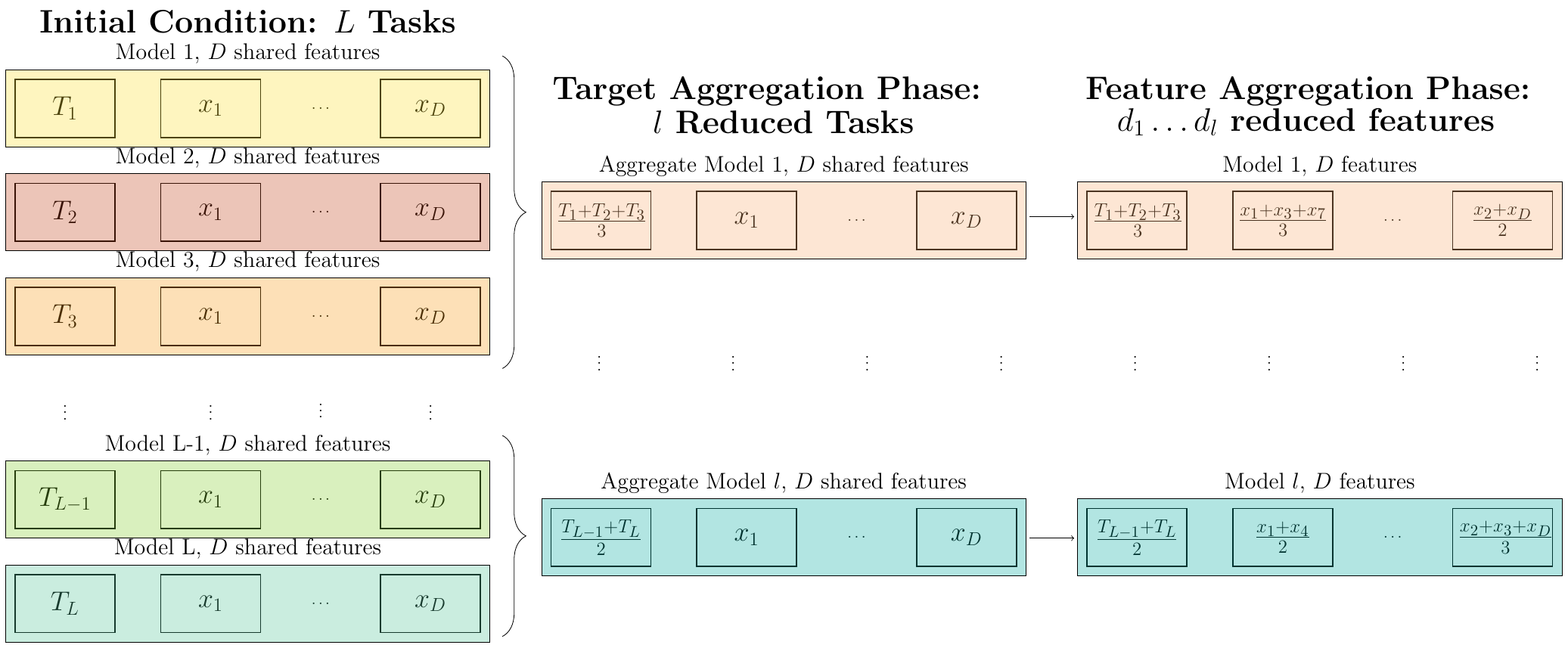}
        \caption{Block scheme of the algorithm: partition of tasks, aggregation of each set of targets with the mean, replication of the procedure to the $D$ features, for each reduced task. 
        }
        \label{fig:workflowMain}
    \end{subfigure}
    \begin{subfigure}{\textwidth}
         \centering
     \includegraphics[width=0.68\textwidth]{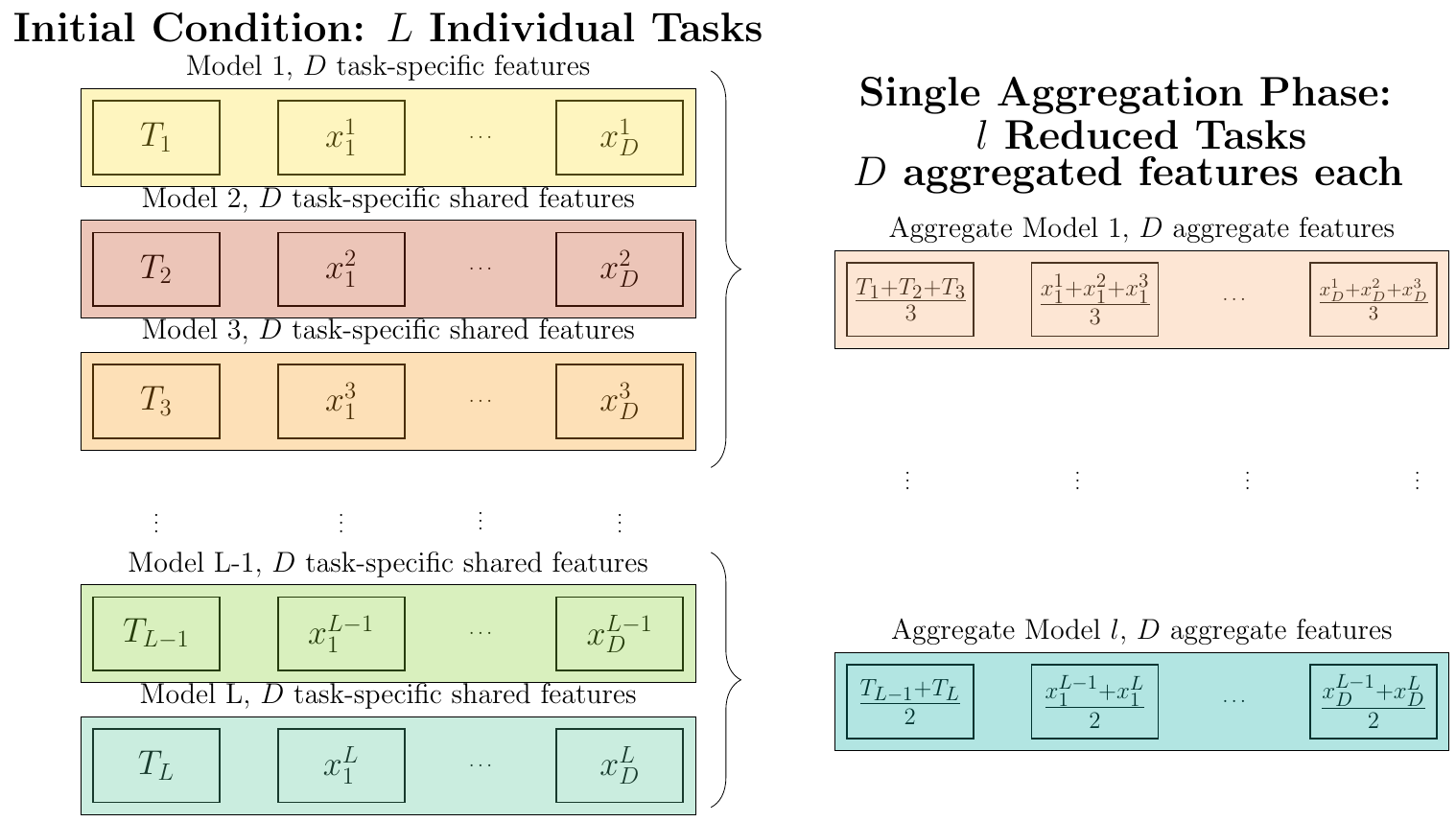}
        \caption{Block scheme of a variant of the algorithm with $D$ features individually measured for each task. In the aggregation phase, both targets and features are averaged.
        }
        \label{fig:workflowVar}
    \end{subfigure}
    \caption{Block scheme of main algorithm and a homogeneous case variant.}
        \label{fig:workflow}
\end{figure}

The choice to consider the mean as an aggregation function is to preserve interpretability, following the definition of interpretability as the property of an ML pipeline to be understood by domain experts, without explanations by ML experts (see~\citep{kovalerchuk2021survey} for a broader discussion). Indeed, we aggregate targets and features, reducing the dimension and the variance of the models and preserving the meaning of each aggregation, which is a mean of variables. On the other hand, the aggregations induce an increase in bias that will be controlled in the analysis.

\begin{figure}
     \centering
     \begin{subfigure}{0.48\textwidth}
         \centering
         \includegraphics[width=\textwidth]{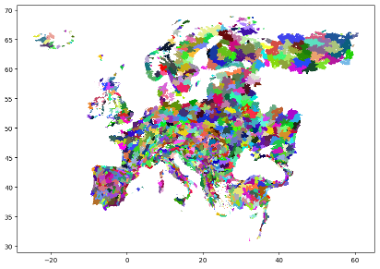}
         \caption{Motivational example.}
         \label{fig:originalBasins}
     \end{subfigure}
     \hfill
     \begin{subfigure}{0.48\textwidth}
         \centering
        \includegraphics[width=\textwidth]{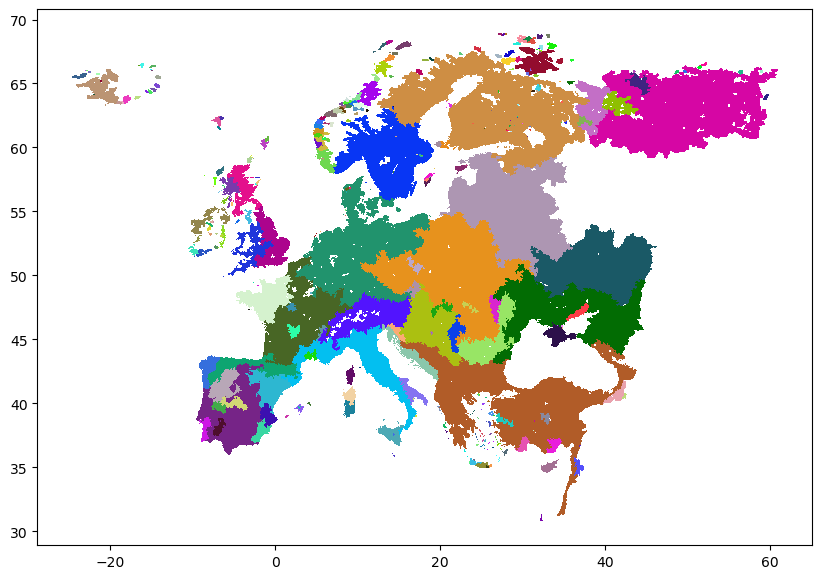}
         \caption{Basins reduced by \algnameshortMain.}
         \label{fig:aggrBasins}
     \end{subfigure}
        \caption{Figure \ref{fig:originalBasins} represents some European sub-regions, identified with different colors, where a target variable and meteorological features are available. Figure \ref{fig:aggrBasins} shows the aggregated regions after the application of \algnameshortMain.}
        \label{fig:introductionExamples}
\end{figure}

A motivational example 
is related to Earth science, where 
we may be interested in the prediction of a target variable at different (correlated) locations, given a set of meteorological features measured in each of them (see Figure \ref{fig:originalBasins}). 
In this context, the proposed algorithm aggregates highly correlated targets and learns an individual model for their mean, resulting in a reduced number of models, simplifying data representation (through further aggregating highly correlated features) and enhancing the performances without loss of interpretability, mitigating overfitting, and limiting the computational complexity (see Figure \ref{fig:aggrBasins}, further discussed in the experimental section). In this example, preserving the interpretability 
is essential for climatologists, filling the gap between advanced ML methods (considered as black-box algorithms) and their physical meaning.


\textbf{Contributions}~~
After introducing notation, problem formulation, and related works (Section \ref{sec:preliminaries}), the contributions of this paper can be summarized as:

$1)$ General theoretical analysis of the asymptotic bias and variance, estimating a target variable with a linear model, trained considering as target an average of targets and as features a reduced set of basis functions (Section \ref{sec:theo}).

$2)$ Introduction of a novel MTL approach, with theoretical guarantees, aggregating targets and features with the mean for interpretability 
    (Section \ref{sec:algo}).
    
$3)$ Extension of recent dimensionality reduction methods~\citep{bonetti2023interpretable,bonetti2023nonlinear}, both to multiple targets and generalizing the theoretical results in the single-task case. 

$4)$ Validation of the proposed algorithm on synthetic data, benchmark MTL datasets and methods, and a multi-task dataset from Earth science (Section \ref{sec:experiments}). 

\noindent Technical proofs and additional results can be found in the Appendix.

\section{Preliminaries}\label{sec:preliminaries}
\textbf{Notation}~~Given $L$ learning tasks $\{\mathcal{T}_i\}_{i=1}^L$, each task $\mathcal{T}_i$ is a supervised learning problem with training set $S_i=\{(\mathbf{x}^i,y_i)_j\}_{j=1}^{n_i}
$, where $n_i$ samples for the $i$-th task are available, $(\mathbf{x}^i)_j=(x_1^i,\dots,x_{D_i}^i)_j\in \mathbb{R}^{D_i}$ is the $j$-th training sample associated to task $i$ composed of $D_i$ features, and $(y_i)_j$ is the corresponding target. The random variable $y_i$ represents the $i$-th target, $x_k^i$ denotes the $k$-th feature of task $\mathcal{T}_i$, and the random vector $\mathbf{x^i}$ is the full set of features of the $i$-th task. The symbols $\sigma^2_{a}$, $cov(a,b)$, $\rho_{ab}$ and $\hat{\sigma}^2_{a}$, $\hat{cov}(a,b)$, $\hat{\rho}_{ab}$ denote the variance of a random variable $a$, its covariance, and correlation with the random variable $b$ and their estimators, respectively. Finally, $\mathbb{E}_a[f(a)]$ is the expected value of a function $f(\cdot)$ of the random variable $a$ \wrt\ its distribution. 

\begin{remark}\label{rem:algorithmsIdea}To simplify the notation, in this paper, 
we mainly consider a shared set of $D$ features $\bm{x}:=\{x_1,\dots,x_D\}:=\{x_1^1,\dots,x_{D_1}^1, \dots, x_1^L,\dots,x_{D_L}^L\}$ and the same number of training samples $n$ for each task. This simplification motivates the workflow of Figure \ref{fig:workflowMain}. We will also discuss a variant with $D$ task-specific features (workflow in Figure \ref{fig:workflowVar}). Recalling the motivational example, this means that we will consider the full set of measurements of meteorological features as a shared set of $D$ features. Then, we will discuss a variant with a set of meteorological features measured in each individual sub-region associated with its target. Additionally, this homogeneous-feature variant could be applied in the heterogeneous settings, aggregating corresponding features (\eg temperature), without changing specific features associated to a sub-region (e.g., some snow-related features may be available only for a mountainous region).  
\end{remark}

\textbf{Data generation process}~~We consider the general relationship between the full set of features and each target with additive Gaussian noise. The noise variables $\{\epsilon_k\}_{k=1}^L$ can be correlated with each other, but independent from features. 

In the theoretical analysis, we consider a partition $\mathcal{P}$ of the tasks. Therefore, each task $\mathcal{T}_k$ belongs to a set of the partition $\mathcal{P}_{\iota(k)}$. 
Without loss of generality, in the analysis we will focus on the $i$-th task $y_i=f_i(x_1\dots x_D) + \epsilon_i,\ \epsilon_i\sim \mathcal{N}(0,\sigma_i^2)$, and we will denote with $\mathcal{P_\iota}:=\mathcal{P}_{\iota(i)}$ the set of $\mathcal{P}$ that contains $\mathcal{T}_i$, and with $K_\iota$ its cardinality. The aggregated (mean) target that contains $y_i$ is therefore $\psi_\iota:=\frac{1}{K_\iota}\sum_{k:y_k \in \mathcal{P}_\iota}y_k$. Finally, to simplify the computations, we assume the expected values of features and targets to be zero: $\mathbb{E}[x_k]=\mathbb{E}[y_i]=\mathbb{E}[f_i(x_1\dots x_D)]=0,\ \forall k\in \{1,\dots,D\},\ \forall i\in \{1,\dots,L\}$.

\textbf{Loss Measure}~~As a natural performance measure for regression and in line with related works~\citep{bonetti2023interpretable,bonetti2023nonlinear}, in the theoretical analysis we evaluate the Mean Squared Error (MSE), focusing on its bias and variance components (bias-variance decomposition~\citep{hastie2009}, see Equation \ref{eq:BiasVarDec} of Appendix \ref{app:variance}).

\textbf{Multi-Task Learning via aggregations}~~
In the theoretical analysis of Section \ref{sec:theo}, we consider a task $y_i=f_i(x_1\dots x_D) + \epsilon_i,\ \epsilon_i\sim \mathcal{N}(0,\sigma_i^2),$ and we evaluate the effect (in terms of MSE) of learning a linear model trained on a target $\psi_\iota$, the mean of a given subset of original targets which contains $y_i$. 
Additionally, 
$d$ given transformations $\{\phi_1,\dots,\phi_d\}$ of the $D$ original features are the inputs. 
They could be any zero-mean transformations (basis functions), but we focus on the means of $d$ subsets of original features. Indeed, through Algorithm \ref{alg:nonLinMTL}, we exploit the theoretical results to identify iteratively, and aggregate with the mean, convenient partitions of features and targets in polynomial time. 

\begin{remark}[Limitations]
    In the analysis, we focus on linear models to derive closed-form results. 
    However, the \algnameshortMain algorithm can be applied 
    with any supervised learning algorithm. In this case, its convenience becomes heuristic, and its effectiveness should be tested empirically. The second main limitation is the asymptoticity of the analysis, 
    to identify expressions of variance and bias 
    without considering confidence intervals of each statistical estimator.
\end{remark}

\subsection{Related Works: Dimensionality Reduction, Multi-Task Learning}

$\quad$\textbf{Dimensionality Reduction}~~
Considering each individual task, the aggregation of its input features can be considered a dimensionality reduction method. These approaches map $D$ features into a reduced dataset of dimension $d<D$, with transformation functions aimed to maximize specific performance measures. Linear, non-linear, supervised, and unsupervised methods exist. For broader discussions and algorithms on dimensionality reduction, we refer to~\citep{vandermaaten2009,AYESHA202044,Jia2022FeatureDR}. The specific idea to aggregate groups of features with their mean resembles group regularization methods, such as Cluster LASSO~\citep{she2008sparse} and OSCAR~\citep{bondell2008simultaneous}, although the method proposed in this paper aggregates features independently from the training of the supervised model. This is in line with some recent approaches~\citep{bonetti2023interpretable,bonetti2023nonlinear}, that we seek to generalize to a multi-task context and by considering the full set of features to identify the aggregations, rather than their bivariate analyses. 

\textbf{Multi-Task Learning}~~The main algorithm of this paper is an MTL approach that identifies subsets of targets, learning a single model for their mean, that we claim to be convenient for the original individual tasks. 
A classical description of MTL is~\citep{caruana1997multitask}, and more recent algorithms are in~\citep{zhang2021survey}. Following their taxonomy, we briefly revise parameter-based and feature-based MTL approaches. A few instance-based approaches also exist, designed to share samples across tasks. 

Parameter-based methods can be based on the assumption that similar tasks have similar model parameters, forcing the coefficient matrix to be low-rank (low-rank approaches, \eg~\citep{Zhang2023LearningLA}); they can estimate and exploit task relationships such as covariances (task-relation approaches, \eg~\citep{Zhou2023AutomaticTR}); they can share sets of parameters (decomposition approaches, \eg~\citep{Li2023MultiTaskLW}); 
or they can assume that different tasks form several clusters (task-clustering approaches, \eg~\citep{Liu2017HierarchicalCM}). Some recent parameter-based MTL algorithms aim to identify groups of related tasks exploiting different metrics, such as the effect of the gradient of the loss associated with one task on the loss of another task~\citep{fifty2021efficiently}, or summary statistics~\citep{,knight2024multi}. The approach of this paper can be considered a task-clustering method since we identify disjoint groups of tasks, averaging the corresponding targets. 

Feature-based methods are based on the assumption that tasks share a common feature representation. Some approaches learn this feature representation by maximizing the information to each task 
(feature transformation). Among them, deep learning-based approaches are the most commonly used (\eg~\citep{ye2022taskprompter}). Other approaches select subsets of original features, maximizing their relevance to the targets (feature selection). They usually optimize a loss function that both penalizes the selection of different features across tasks and minimizes the number of important features (\eg~\citep{Zhong2023HeterogeneousMF}). The proposed algorithm combines the task-clustering phase with a feature transformation phase, where for each cluster of tasks, we identify and aggregate subsets of features. 



\section{Bias-Variance Analysis: Theoretical Results}\label{sec:theo}

In this section, we present the main theoretical results, 
deriving the asymptotic bias and variance of aggregated tasks. We consider a generic task $\mathcal{T}_i$: 
\begin{equation}\label{eq:taski}
    y_i=f_i(x_1\dots x_D) + \epsilon_i,\ \epsilon_i\sim \mathcal{N}(0,\sigma_i^2),
    \end{equation} 
    and the set of targets $\mathcal{P_\iota}$, that contains $y_i$. In this setting, we estimate the target $y_i$ with a linear model, trained with the mean of the elements of $\mathcal{P_\iota}$ as target, and $d$ inputs $\{\phi_1,\dots,\phi_d\}$. 
    Recalling that $\psi_\iota:=\frac{1}{K_\iota}\sum_{k:y_k \in \mathcal{P}_\iota}y_k$, we estimate: 

\begin{equation}\label{eq:estimatedModelGeneral}
\hat{y}_i=\hat{\psi}_\iota=
\hat{w}_1^\iota \phi_1 + \dots + \hat{w}_d^\iota \phi_d,\end{equation} 

with $\hat{w}_j^\iota$ least squares estimates. In particular $\{\phi_1,\dots,\phi_d\}=\{x_1,\dots,x_D\}$ is the case with original features, and $\mathcal{P}_\iota=\{y_i\}$ identifies the single-task case. 


\textbf{Variance}~~We firstly derive the variance of the linear model $\hat{\psi}_\iota$ (Equation \ref{eq:estimatedModelGeneral}), for the $i$-th target $y_i$. Proofs and additional discussions are in Appendix \ref{app:variance}.

\begin{theorem}
\label{thm:variance}
Let the relationship between the features and the target of a task $\mathcal{T}_i$ be defined as Equation \ref{eq:taski}. Let also each estimator converge in probability to the quantity that it estimates. In the asymptotic case, let $var_d^\iota$ be the variance of a linear regression trained with the basis functions $\{\phi_1,\dots,\phi_d\}$ as inputs and the mean of a cluster of targets $\psi_\iota$, defined in Equation \ref{eq:estimatedModelGeneral}, as output. It holds:
\begin{equation}\label{eq:variance}
    var_d^\iota = \frac{\bar{\sigma}_\iota^2}{(n-1)}\cdot d,
\end{equation}
where $\bar{\sigma}_\iota^2:=var(\frac{1}{K_\iota}\sum_{k:y_k \in \mathcal{P}_\iota} \epsilon_k)$ is the variance of the mean of noises of the targets in $\mathcal{P}_\iota$, $n$ is the number of training samples, and $d$ is the number of inputs. 
\end{theorem}

\begin{remark}
    The theorem follows the intuition that the asymptotic variance is proportional to the number of inputs $d$ (i.e., the number of estimated coefficients) 
    and the variance gets smaller as the number of samples increases, 
    since the estimate is more accurate. Finally, the asymptotic variance increases with $\bar{\sigma}_\iota^2$, which is the distortion associated with the averaged target $\psi_\iota$. 
\end{remark}

\begin{remark}
    Some specific cases are particularly relevant.
    
    $1)$ When the inputs are the $D$ original features, we get $var_D^\iota = \frac{\bar{\sigma}_\iota^2}{(n-1)}\cdot D$. For a single-task approach, the variance is $var_d^i =  \frac{{\sigma}_i^2}{(n-1)}\cdot d$. Combining them, 
    ($D$ features and a single-task), 
    the variance is $var_D^i = \frac{{\sigma}^2_i}{(n-1)}\cdot D$.
    
    $2)$ When the variance of the noises is constant ($\forall i:\sigma^2_i=\sigma^2$), we get: $\bar{\sigma}_\iota^2 = \frac{\sigma^2}{K_\iota}(1+(K_\iota-1)\bar{\rho}_{\iota})$, with $\bar{\rho}_{\iota}$ average correlation among noises of the targets in 
    $\mathcal{P}_\iota$. 
    Therefore, if $\bar{\rho}_\iota=1$, the task aggregation does not reduce the variance since the variance of the averaged noise is equal to the individual ($\sigma^2$). 
    On the contrary, since $\bar{\rho}_\iota\geq \frac{1}{1-K_\iota}$ by non-negativity of variance, maximum variance gain corresponds to the minimum average correlation among noises $\bar{\rho}_\iota = \frac{1}{1-K_\iota}$. 
    Finally, when the noises are independent ($\bar{\rho}_\iota=0$), the aggregated noise variance is reduced by a factor $K_\iota$ ($\frac{1}{K_\iota}\sigma^2$). Intuitively, when the noises are less correlated, the different tasks are better exploited to refine the knowledge about each task. 
    
    $3)$ In~\citep{bonetti2023interpretable,bonetti2023nonlinear} a similar asymptotic result is provided in the single-task setting, with an asymptotic variance equal to $\frac{2\sigma^2_i}{(n-1)}$ in the bivariate case and $\frac{\sigma^2_i}{(n-1)}$ in the univariate one. Theorem \ref{thm:variance} generalizes that result to $d$ input features, revealing a cost of $\frac{\sigma^2}{n-1}$ for each feature considered. Therefore, by reducing the features from $D$ to $d<D$, the asymptotic variance decrease is $\frac{\sigma^2}{n-1}\cdot(D-d)$.
\end{remark}

\textbf{Bias}~~We now focus on the bias of the linear model $\hat{\psi}_\iota$ (Equation \ref{eq:estimatedModelGeneral}), \wrt the $i$-th target $y_i$. Proofs and more discussions can be found in Appendix \ref{app:bias}.

\begin{theorem}
\label{thm:bias}
Let the relationship between the features and the target of a task $\mathcal{T}_i$ be defined as Equation \ref{eq:taski}. Let also each estimator converge in probability to the quantity that it estimates. In the asymptotic case, let $bias_d^\iota$ be the bias of a linear regression trained with the basis functions $\{\phi_1,\dots,\phi_d\}$ as inputs and the mean of a cluster of targets $\psi_\iota$, defined in Equation \ref{eq:estimatedModelGeneral}, as output. It is equal to:

\begin{equation}\label{eq:bias}
\begin{split}
    bias_d^\iota 
    & = \sigma^2_{f_i} - \sigma^2_{\psi_\iota}R^2_{d,\iota}  +2 (cov(\psi_\iota,f_i - \psi_\iota | \mathbf{\Phi}) - cov(\psi_\iota,f_i - \psi_\iota)),
\end{split}
\end{equation}
with $\sigma^2_{f_i}$ variance of the function of inputs defining the $i$-th task, $\sigma^2_{\psi_\iota}$ variance of the aggregated target $\psi_\iota$, $R^2_{d,\iota}$ squared coefficient of multiple correlation between the $d$ inputs and the aggregated target $\psi_\iota$, and $cov(\psi_\iota,f_i - \psi_\iota | \mathbf{\Phi})$ partial covariance between the two random variables $
\psi_\iota,f_i - \psi_\iota$, given the inputs $\mathbf{\Phi}$. 
\end{theorem}

\begin{corollary}\label{cor:biasSingle}
    Considering the single-task setting, \ie assuming $\psi_\iota=y_i$, the asymptotic bias of Theorem \ref{thm:bias} reduces to:
    \begin{equation}\label{biasSingle}
        bias_{d}^i = \sigma^2_{f_i} (1-R^2_{d,i}).
    \end{equation}
\end{corollary}

\begin{remark} Intuitively, Corollary \ref{cor:biasSingle} shows that, in single-task problems, the bias is proportional to the information, measured through the coefficient of multiple correlation, that the inputs share with the target. Additionally, the bias is proportional to the variance of the function $f_i$ that regulates the $i$-th data generation process. More generally, Theorem \ref{thm:bias} shows that estimating the $i$-th target with a linear model trained with the aggregated target $\psi_\iota$, the bias is still proportional to the variance of the original $i$-th target, which is an irreducible cost. 
Then, the bias decreases proportionally to the coefficient of multiple correlation between aggregated target and features, representing the skill of the linear model to predict the aggregated target, weighted by the variance of the aggregated target itself, which accounts for the simplification introduced by aggregating. 
    Finally, the third term of Equation \ref{eq:bias} is the difference of the covariance, with and without conditioning on the inputs, between the aggregated target and the gap between the $i$-th task and the aggregated target itself. This way, if the features reduce the information shared between $\psi_\iota$ and its gap with $f_i$, the bias reduces since we are exploiting the features to improve the learning of the $i$-th task. 
\end{remark}
\begin{remark}
    Some specific cases can also be derived from the results on the bias.
    
    $1)$ Considering the $D$ original features, the results of Theorem \ref{thm:bias} and Corollary \ref{cor:biasSingle} hold, evaluating the quantities of the expressions with them (\eg $R^2_{D,\iota}$). 
    
$2)$ In~\citep{bonetti2023interpretable,bonetti2023nonlinear} a similar (bivariate) asymptotic analysis is provided in the single-task setting. Corollary \ref{cor:biasSingle} extends those findings to a general case with $d$ inputs. Therefore, if we perform a single-task dimensionality reduction, aggregating $d$ sets of features with their mean, the asymptotic bias variation is 
$\sigma^2_{f_i} (R^2_{D,i}-R^2_{d,i})$. 
\end{remark}

\textbf{Theoretical bounds for aggregations}~~We conclude this Section by showing a condition for a convenient aggregation of two tasks and another for features in single-task problems. Both results can be deduced from Theorem \ref{thm:variance} and \ref{thm:bias}, and they will be exploited in the two phases of the main algorithm, respectively.

\begin{corollary}
    Considering two tasks $\mathcal{T}_i,\mathcal{T}_j$ regulated by Equation \ref{eq:taski}, 
    training an individual linear model with the mean of the two targets as output 
    is profitable 
    \wrt the individual single-task models, in terms of MSE, 
    if and only if:
    \begin{equation}\label{eq:profitable2task}
        \begin{cases}
            \frac{{\sigma}^2_i}{(n-1)}\cdot D + \sigma^2_{\psi_\iota}R^2_{D,\iota}\geq \frac{{\sigma}^2_\iota}{(n-1)}\cdot D + \frac{1}{2}[\sigma^2_{f_i}R^2_{D,i}+\sigma^2_{f_j}R^2_{D,j}]\\
            \frac{{\sigma}^2_j}{(n-1)}\cdot D + \sigma^2_{\psi_\iota}R^2_{D,\iota}\geq \frac{{\sigma}^2_\iota}{(n-1)}\cdot D + \frac{1}{2}[\sigma^2_{f_i}R^2_{D,i}+\sigma^2_{f_j}R^2_{D,j}].
        \end{cases}
    \end{equation}
\begin{proof}
    We compute variance and bias from Theorem \ref{thm:variance} and \ref{thm:bias}, substituting $\psi_\iota=y_i$, $\psi_\iota=y_j$, and $\psi_\iota=\frac{y_i+y_j}{2}$. 
    Then, we impose that the sum of variance and bias of the aggregated case is not worse than the individual one for both models. 
\end{proof}
\end{corollary}
Following the intuition, Equation \ref{eq:profitable2task} shows the convenience of aggregating two targets if the variance of the noise is reduced or when the predictive capability of the aggregated model 
is better than the average of the single-task ones.

\begin{corollary}
    Considering a task $\mathcal{T}_i$, the aggregation of $D$ features into $d<D$ aggregated ones is profitable, in terms of MSE of a linear model, if and only if: 
    \begin{equation}\label{eq:NonLinCFAreformulation}
        \frac{\sigma^2_i}{(n-1)}\cdot (D-d) \geq \sigma^2_{f_i} (R^2_{D,i}-R^2_{d,i}).
    \end{equation}
\begin{proof}
    The result follows comparing the single-task bias and variance that are particular cases of the general results of Theorem \ref{thm:variance} and \ref{thm:bias}.
\end{proof}
\end{corollary}


Following the intuition, Equation \ref{eq:NonLinCFAreformulation} shows the convenience in terms of variance to reduce the number of features or in terms of bias when the predictive capability of the linear regression 
is better when the features are aggregated.

\section{Multi-Task Learning via Aggregations: Algorithms}\label{sec:algo}

In this section, we present the \algnameshortMain algorithm, assuming $L$ tasks and $D$ shared features. The algorithm is depicted in Figure \ref{fig:workflowMain}, while a multi-input variant, discussed in Remark \ref{rem:algorithmsIdea}, is depicted in Figure \ref{fig:workflowVar}. 
Algorithm \ref{alg:nonLinMTL} reports the pseudo-code of \algnameshortMain, exploiting the loop of Algorithm \ref{alg:AggregationLoop} (in Appendix \ref{app:algo}) to iteratively aggregate targets firstly, and features subsequently.

\textbf{Phase I: task-aggregation.}~~
Firstly, the algorithm iteratively adds targets to a set, forming a partition until no aggregation is convenient. At any iteration, Equation \ref{eq:profitable2task} is exploited to test the convenience of adding a target variable $y_j$ into a set of the partition, initialized as a singleton of a random target. Additionally, a hyperparameter $\epsilon_1$ regulates the propensity of the algorithm to aggregate.

Equation \ref{eq:profitable2task} regulates the convenience of an aggregation of two targets. Therefore, the algorithm starts with $L$ individual tasks, and it identifies two tasks that benefit from the aggregation (\eg $\mathcal{T}_1,\mathcal{T}_2$), moving to $L-1$ linear models. Then, the algorithm further aggregates a third task (\eg $\mathcal{T}_3$) with them, moving to $L-2$ linear models, if the new aggregation ($\frac{y_1+y_2+y_3}{3}$ in the example) is convenient \wrt the aggregate two-task model ($\frac{y_1+y_2}{2}$) and the univariate one ($y_3$). The new aggregation is therefore convenient \wrt the original individual tasks, and it is further convenient \wrt the previous two-task aggregation.

 Similar to forward feature selection, we add a target in the current set, if convenient, without inspecting all the possible combinations, which would be combinatorial. This way, we do not identify the \emph{optimal} partition of targets, but a convenient one \wrt the single-tasks, with a quadratic number of comparisons in the number of tasks ($\mathcal{O}(L^2)$) in the worst case, \ie with no aggregations. 
 
 Additionally, the aggregation depends on the ordering of targets. For this reason, we randomize it to avoid systemic biases. A possible variation could be to introduce a heuristic (\eg the correlation between couples of targets) to rank them and test for the aggregations based on this ranking.  

 \textbf{Phase II: cluster-level feature-aggregation.}~~In the second phase, \algnameshortMain identifies groups of features to aggregate with their mean for each of the $l$ reduced tasks. 
 This way, for each reduced task $\iota$, we identify $d_\iota$ reduced features, 
 reducing the dimension and improving the performance.

 Specifically, we exploit Equation \ref{eq:NonLinCFAreformulation} to iteratively identify couples of features that are convenient to average, following the same iterative procedure of the target aggregation phase. 
 This is quadratic \wrt the number of comparisons for each aggregated task ($\mathcal{O}(l\cdot D^2)$). Additionally, since 
 the terms $\sigma^2_i,\sigma^2_{f_i},n$ are constant across different comparisons, we include them in the hyperparameter $\epsilon_2$, which regulates the propensity of the algorithm to aggregate features.

\begin{algorithm}[ht]
\caption{\algnameshortMain:\algnameMain}\label{alg:nonLinMTL}
\begin{algorithmic}
\Require{features $\bm{x}=\{x_1\dots x_D\}$; targets $\bm{y}=\{y_1\dots y_L\}$; $n$ samples, tolerances $\epsilon_1,\epsilon_2$}
\Ensure{reduced tasks $\{\psi_1,\dots,\psi_l\}$, reduced features $\{\phi_1,\dots,\phi_{d_\iota}\}_{\iota=1}^l$}\\ 
\vspace{-0,3cm}
\Function{\textsc{Compute\_threshold\_features}$( \bm{x}_{curr},y,z_\mathcal{P},z_j,\epsilon)$}{}\Comment{From Equation \ref{eq:NonLinCFAreformulation}}
            \State  $R_{sep} \leftarrow \text{R2score}(\bm{x}_{curr}, y)$
            \State $R_{aggr} \leftarrow \text{R2score}((\bm{x}_{curr}\setminus \{z_\mathcal{P}, z_j\}) \cup \{mean(z_\mathcal{P},z_j)\}, y)$\\
        \Return{$R_{sep}-R_{aggr} \leq \epsilon$}
        \EndFunction
        \Statex
\vspace{-0,3cm}
\Function{\textsc{Compute\_threshold\_targets}$( \bm{x},y_\mathcal{P},y_j,\epsilon )$}{}\Comment{From Equation \ref{eq:profitable2task}}
            \State $y_{ag} \leftarrow mean(y_\mathcal{P},y_j)$
            \State  $R_{\mathcal{P}}, \sigma^2_{\mathcal{P}}, \sigma^2_{f_\mathcal{P}} \leftarrow \text{R2score}(\bm{x},y_\mathcal{P}), \text{var\_res}(\bm{x},y_\mathcal{P}),
            \text{var}(y_\mathcal{P})-\text{var\_res}(\bm{x},y_\mathcal{P})$
            \State $R_j, \sigma^2_{j}, \sigma^2_{f_j} \leftarrow \text{R2score}(\bm{x},y_j), \text{var\_res}(\bm{x},y_j),
            \text{var}(y_j)-\text{var\_res}(\bm{x},y_j)$
            \State  $R_{ag}, \sigma^2_{ag}, \sigma^2_{f_{ag}} \leftarrow \text{R2score}(\bm{x},y_{ag}), \text{var\_res}(\bm{x},y_{ag}), \text{var}(y_{ag})-\text{var\_res}(\bm{x},y_{ag})$\\
            \vspace{-0,3cm}
        \State $threshold_1 =\frac{D}{(n-1)}(\sigma^2_{ag}-\sigma^2_{\mathcal{P}}) + \frac{1}{2}(R_{\mathcal{P}} \sigma^2_{f_{\mathcal{P}}} + R_j \sigma^2_{f_j}) - R_{ag} \sigma^2_{f_{ag}}$\\
        \vspace{-0,3cm}
        \State $threshold_2 =\frac{D}{(n-1)}(\sigma^2_{ag}-\sigma^2_{j}) + \frac{1}{2}(R_{\mathcal{P}} \sigma^2_{f_{\mathcal{P}}} + R_j \sigma^2_{f_j}) - R_{ag} \sigma^2_{f_{ag}}$\\
        \Return{$(threshold_1 \leq \epsilon)\ \mathbf{AND}\ (threshold_2 \leq \epsilon)$}
        \EndFunction
        \Statex
        
\Function{\textsc{\algnameshortMain}$( Input)$}{}\Comment{Main function}
\State{PHASE I: task aggregations}
\State $\{\psi_1,\dots,\psi_l\} \leftarrow $ \textsc{Aggregation}$(\bm{z}=\bm{y},phase=1, \epsilon=\epsilon_1,\bm{x}=\bm{x},y=None)$
\State{PHASE II: feature aggregation for each task}
\ForEach {$\iota \in \{ 1,\dots,l\}$}
\State $\{\phi_1,\dots,\phi_{d_\iota}\} \leftarrow $ \textsc{Aggregation}$(\bm{z}=\bm{x},phase=2, \epsilon=\epsilon_2,\bm{x}=None,y=\psi_\iota)$
\EndFor\\
\Return $\{\psi_1,\dots,\psi_l\}, \{\phi_1,\dots,\phi_{d_\iota}\}_{\iota=1}^l$
\EndFunction
\end{algorithmic}
\end{algorithm}

\begin{remark}
    Algorithm \ref{alg:nonLinMTL} outputs a set of reduced tasks and the associated sets of reduced features. We do not include a final regression 
    to decouple the algorithm from the regression model, which can be run independently, with theoretical guarantees in linear regression. 
    In this sense, we propose a \emph{filter} method. 
\end{remark}

\begin{remark}\label{rem:variants}
    Recalling Remark \ref{rem:algorithmsIdea}, we may want to consider a multi-input setting, with $D$ features for each task individually 
    $\{x_1^i,\dots,x_D^i\}_{i=1}^L$. In our example, they can be $D$ meteorological features measured at each location and associated with its specific target. A variant of Algorithm \ref{alg:nonLinMTL} (Figure \ref{fig:workflowVar}) 
    compares, at each iteration, the model trained with the features associated with each individual target and the model for the aggregated task, trained on averaged features. 
    This way, in a single phase, we aggregate couples of targets and pairwise couples of features. In our example, we would compare individual measurements of temperature and precipitation for single-task problems and the means of temperatures and precipitations as the two inputs of the averaged model. This variant can be extended to a heterogeneous case where, 
    for example, the first task has a snow-related feature not available for the other. In this case, we consider it when the first task appears, 
    keeping unchanged the aggregation of targets and other features.
\end{remark}

\section{Experimental Validation}\label{sec:experiments}
In this section, we show synthetic experiments, validating the proposed algorithm \wrt single-task regressions, showing its behavior varying parameters, and with an ablation study of its two phases. Then, applications to real-world data show the competitiveness of the method \wrt single task and benchmark MTL approaches. Code can be found at: \url{https://github.com/PaoloBonettiPolimi/NonLinCTFA}.

\subsection{Synthetic Experiments and Ablation Study}\label{sub:synthExp}
We start with synthetic experiments, validating the \algnameshortMain against single tasks. 
Specifically, we consider $L=10$ tasks, $D=100$ features shared across all tasks, $n=250$ training samples (same number for testing), the standard deviation of (independent) noises $\sigma=10$, and hyperparameters $\epsilon_1=0, \epsilon_2=0.0001$. Each target is obtained as a linear combination of all the features, with additive Gaussian noise, randomly sampling each coefficient from a uniform distribution in the interval $[-1,-0.5]$ or $[0.5,1]$, obtaining two groups of tasks similar among themselves. 
We perform linear regression on individual tasks and on aggregated tasks after the first phase of the algorithm or fully applying \algnameshortMain. We repeat the experiment $10$ times to produce confidence intervals, and we consider as metrics the MSE and the coefficient of determination ($R^2$), both in terms of absolute values and of percentage increase \wrt single-task. We obtain a single-task average $R^2$ score of $0.48\pm 0.02$, which increases to $0.64\pm 0.01 (+33.45\pm3.31\%)$ considering the aggregated tasks and $0.67\pm0.01 (+39.82\pm3.42\%)$ adding the feature aggregation. Similarly, the MSE is equal to $9.34\pm0.04$, reducing to $6.54\pm0.03 (-29.44\pm1.45\%)$ and $5.98\pm0.06 (-35.36\pm1.27\%)$. The $L=10$ tasks become $l=2.5\pm0.6$, and the $D=100$ features reduce to $d=3.43\pm1.76$. These results empirically validate the improvement provided by the proposed algorithm \wrt the single-task counterparts in linear regression. Additionally, they show a significant benefit with task-aggregation, with a subsequent feature-aggregation phase refining the performances, with the added value of simplifying models. 

\begin{figure}
\includegraphics[width=\textwidth]{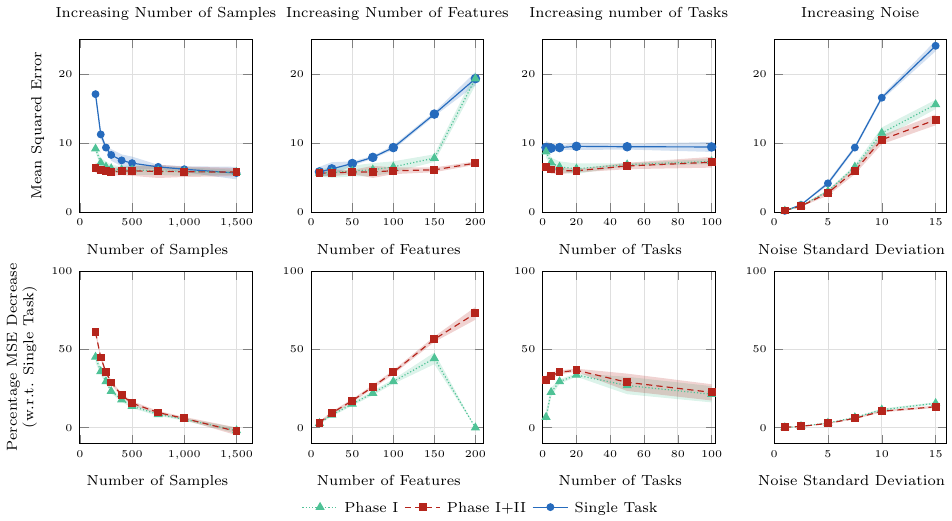}
\caption{Test MSE (first row) and corresponding percentage decrease \wrt single-task (second row), varying one parameter at a time, only aggregating targets (Phase I) or adding the feature-aggregation phase (Phase I+II). 
} \label{fig:synthMSE}
\end{figure}

We further tested the proposed approach, varying one parameter at a time and fixing the others. Figure \ref{fig:synthMSE} reports the results, in terms of MSE and of percentage increase of MSE \wrt the single-task, showing the improvement with task-aggregation and then the combination with feature aggregation. The first column of the figure shows that, with a small number of samples, both phases provide significant improvement, which is mitigated by a large number of samples. In the second column, when the number of features increases, the feature aggregation phase is more relevant. Similarly, the third column shows that the task-aggregation phase is more relevant when the number of tasks increases, assuming a constant number of features. Finally, the fourth column shows that the problem becomes more difficult when the noise of the targets increases, with the task aggregations that become more relevant. 
Figure \ref{fig:synthMSE2} in Appendix \ref{app:synth} also shows the behavior of the MSE when the hyperparameters $\epsilon_1,\epsilon_2$ are varied. In particular, $\epsilon_1$ represents the propensity to aggregate tasks, spanning from the single-task case to a single aggregated task, with a value equal to $0$ that balances the aggregation. 
Similarly, $\epsilon_2$ regulates the propensity to aggregate features, spanning from preserving all original inputs to a single aggregation. 

\subsection{Real World Datasets}
As a first investigation on real data, we consider two classical datasets, 
comparing single-task regression performance, applying \algnameshortMain, and with some standard MTL approaches. In particular, the SARCOS dataset~\citep{vijayakumar2005incremental} is composed of $7$ regression tasks, $21$ shared features, and $n=48933$ samples. We consider $n=1000$ and the full set of $n=48933$ samples to investigate the effect of the number of samples. Then, the \emph{School} dataset~\citep{bakker2003task} consists of $15362$ samples, exam records, distributed across $139$ tasks (schools), with $27$ features. This is an example where different measures of the same features are associated to their tasks, allowing to apply the variant of \algnameshortMain discussed in Remark \ref{rem:variants}. 
In this case, we consider $27$ tasks and $27$ features, as well as the full set of $139$ tasks. 

We apply \algnameshortMain, combined with linear regression (LR), support vector regression (SVR), and multi-layer perceptron (MLP), in comparison with the corresponding single-task models. Additionally, we consider a random prediction as a trivial baseline and Alternating Structural Optimization~\citep[ASO]{ando2005framework}, its convex relaxation~\citep[cASO]{chen2009convex}, and Convex Clustered MTL~\citep[CMTL]{jacob2008clustered}, as benchmarks representing feature-based and clustered-based MTL methods (adapting the implementation of \url{https://github.com/chcorbi/MultiTaskLearning}). In line with the implementation of these baselines, we evaluate test performance with normalized root mean squared error (NRMSE), randomly selecting the $30\%$ of samples, with five different seeds, to produce test confidence intervals. 

\begin{table}[t]
\caption{Experiments on SARCOS and \emph{School} datasets, considering 70\% of data for training and 30\% for testing, randomizing over $5$ seeds for confidence intervals. NRMSE is the performance measure (lower is better, best result in bold).}
\label{tab:realRes}
\vspace{0.2cm}
\centering 
\resizebox{\textwidth}{!}
{
{\begin{tabular}{@{}ccccc@{}} \hline 
 & SARCOS\_1000samples & SARCOS\_full & School\_27tasks & School\_full \\\hline 
\# samples $n$ & $1000$ & $48933$ & $15362$ & $15362$\\
\hline
\# tasks $L$ & $7$ & $7$ & $27$ & $193$\\
Reduced \# tasks (\textbf{ours})  & $6.0\pm0.0$ & $6.0\pm0.0$ & $6.0\pm0.0$ & $23.2\pm2.7$\\
\hline
\# features $D$ & $21$ & $21$ & $27$ (for each task) & $27$ (for each task)\\
Reduced \# features (\textbf{ours})  & $7.1\pm0.3$ & $6.9\pm0.4$ & $27$ (for each aggregation) & $27$ (for each aggregation)\\
\hline
NRMSE random & $0.363 \pm 0.015$ & $0.235\pm 0.011$ & $0.634\pm0.123$ & $0.641\pm0.294$ \\
NRMSE single-task LR & $0.085 \pm 0.001$ & $0.069\pm0.002$ & $0.183\pm0.002$ & $0.165\pm0.003$ \\
NRMSE single-task SVR & $0.142\pm0.010$ & $0.092\pm0.008$ & $0.181\pm0.009$ & $0.162\pm0.005$\\
NRMSE single-task MLP & $0.069\pm0.002$ & $0.045\pm0.003$ & $0.182\pm0.003$ & $0.174\pm0.008$ \\
\hline
NRMSE ASO  & $0.075\pm0.002$ & $0.049\pm0.001 $ & $0.172\pm 0.003$ & $\mathbf{0.152\pm0.006}$ \\
NRMSE cASO  & $0.068\pm0.001$ & $0.048\pm0.001$ & $0.173\pm0.002$ & $0.154\pm0.006$ \\
NRMSE CMTL  & $0.111\pm0.001$ & $0.067\pm0.002$ & $0.843\pm0.316$ & $1.187\pm0.562$ \\
\hline 
NRMSE \algnameshortMain + LR (\textbf{ours}) & $0.054\pm0.003$ & $0.035\pm0.002$ & $0.162\pm0.005$ & $0.159\pm0.002$\\
NRMSE \algnameshortMain + SVR (\textbf{ours}) & $0.154\pm0.021$ & $0.115\pm0.024$ & $\mathbf{0.159\pm0.003}$ & $0.155\pm0.004$ \\
NRMSE \algnameshortMain + MLP (\textbf{ours}) & $\mathbf{0.049\pm0.009}$ & $\mathbf{0.031\pm0.001}$ & $0.160\pm0.004$ & $0.167\pm0.003$\\
\hline
\end{tabular}}
}
\end{table}

The results of Table \ref{tab:realRes} show that \algnameshortMain outperforms single-task models and is competitive \wrt MTL baselines, with the advantage of reducing the number of models and parameters, still preserving the interpretability.

In a second real-world application, we consider the QM9 dataset~\citep{ramakrishnan2014quantum}, a challenging quantum chemistry dataset with $L=19$ tasks (properties of molecules), with $139000$ graph inputs (molecules structures). We averaged node features, position, and edge attributes, obtaining $D=19$ features. Following the experimental setup of~\citep{navon2022multi}, we retrieved the dataset from PyTorch Geometric, randomly selecting $10000$ samples for testing and the others for training, normalizing each task and repeating the experiments three times to produce confidence intervals. 

\begin{table}[t]
\caption{Experiments on QM9 dataset. $10000$ random samples are used for testing ($\sim138000$ for training), $3$ different seeds for confidence intervals. MSE is the performance measure (lower is better). Training time is also reported.}
\label{tab:realRes2}
\vspace{0.2cm}
\centering 
\resizebox{\textwidth}{!}
{
\begin{tabular}{@{}ccccc@{}} \hline 
 QM9 Test Results & \# reduced tasks & \# reduced features & MSE & Time (minutes) \\\hline 
Single-task LR & $19$ & $20$ & $0.969\pm0.049$ & $\sim 1$ \\
Single-task MLP  & $19$ & $20$ & $0.518\pm0.032$ & $\sim 11$\\
\hline
Single-task baseline of~\citep{Tong2022LearnableFF}   & $19$ & $13$ & $0.533\pm0.041$ & NA \\
HPS GNN + RLW  & $19$ & $11+$graph & $0.619\pm0.254$ & $\sim1\times300$epochs \\
Best GNN of~\citep{Tong2022LearnableFF}  & $19$ & $11+$graph & $0.216\pm0.009$ & NA \\
\hline
NonLinCTFA + LR (\textbf{ours})    & $12\pm1.2$ & $10.05\pm2.96$ & $0.955\pm0.038$ & $\sim 2$\\
NonLinCTFA + MLP (\textbf{ours})   & $12\pm1.2$ & $10.05\pm2.96$ & $0.469\pm0.024$ & $\sim14$\\
\hline
\end{tabular}
}
\end{table}

Table \ref{tab:realRes2} shows the performance, in terms of test MSE, of the single-task linear regression (LR) and multi-layer perceptron (MLP), together with the MSE associated to the same models, applying \algnameshortMain. Additionally, we exploited the implementation of LibMTL~\citep{Lin2022LibMTLAP} of some state-of-the-art MTL methods and its integration with the QM9 dataset for further benchmarking. In particular, the architecture of the library compatible with graph neural networks is the hard parameter sharing (HPS GNN) that we trained with all the $14$ weighting strategies implemented in the library (we refer to it for details), up to $300$ epochs, identifying the Random Loss Weighting strategy~\citep[RLW]{lin2021reasonable} as best performing. For further comparison with the literature, in the table, we also show the MSE of the baseline and the best-performing graph neural network (GNN) proposed in~\citep{Tong2022LearnableFF}. Together with the MSE, Table \ref{tab:realRes2} shows the computational time. We can conclude that the \algnameshortMain provides significant aggregations, improving single task performances, especially combined with the MLP. Additionally, it is competitive \wrt GNN-based approaches, without outperforming all of them given its much simpler tabular methodology, as also highlighted by the computational time. 

A final experimental setup shows an application of the \algnameshortMain Algorithm to meteorological data, as depicted in Figure \ref{fig:introductionExamples}, where $L=29934$ European hydrological basins are considered as tasks, each with a satellite signal as target and $D=16$ meteorological measurements and climate indices as inputs. In this context, we apply the variant of the algorithm described in Remark \ref{rem:variants}, with linear regression (given the small amount of monthly measurements $n=102$ and the necessity to preserve the interpertability of the entire workflow). 

\begin{table}[t]
\caption{Experiments on climate dataset with \algnameshortMain, varying the hyperparameter $\epsilon$. The number of aggregations and MSE are obtained cross-validating. 
}
\label{tab:realRes3}
\vspace{0.2cm}
\centering 
\resizebox{\textwidth}{!}
{
\begin{tabular}{@{}cccccc@{}} \hline 
 \algnameshortMain & $\epsilon=1$ (single-task) & $\epsilon=0.5$ & $\epsilon=0.1$ & $\epsilon=0.05$ & $\epsilon=0.01$ \\
 \# reduced tasks & $29934$ & $28024$ & $2354$ & $1345$ & $969$ \\
 MSE & $1.058\pm0.224$ & $1.045\pm0.221$ & $0.758\pm0.148$ & $0.755\pm0.143$ & $0.753\pm0.143$ \\
\hline
& $\epsilon=0$ & $\epsilon=-0.01$ & $\epsilon=-0.05$ & $\epsilon=-0.1$ & $\epsilon=-1$ \\
 \# reduced tasks & $944$ & $844$ & $680$ & $252$ & $1$ \\
 MSE & $0.750\pm0.140$ & $0.746\pm0.142$ & $0.756\pm0.146$ & $0.909\pm0.189$ & $1.132\pm0.193$\\
 \hline
\end{tabular}
}
\end{table}

Table \ref{tab:realRes3} shows confidence intervals, in terms of MSE, associated with different values of the hyperparameter $\epsilon$. As expected, the MSE on the original tasks reduces when reducing the value of $\epsilon$ since the aggregations of tasks are convenient. However, when the hyperparameter is too small, the algorithm aggregates too many tasks, becoming detrimental to the understanding of the behavior of the original tasks until the limit case of a single aggregated task. The average number of reduced tasks is also reported in the table to confirm this behavior.

\section{Conclusions and Future Developments}\label{sec:conclusions}
In this paper, we introduced a two-phase MTL approach that aggregates sets of targets and features with their mean, motivated by meteorological applications, and aimed to improve the final regression performance, preserving interpretability. We provided a bias-variance analysis, considering linear regression, and we empirically validated the approach with synthetic and real-world datasets, showing promising results also outside the context of linear regression. A future development can be an extension of the analysis to general ML models. Additionally, an applicative work with meteorological data is under development.

\bibliographystyle{splncs04}
\bibliography{bibliography.bib}

\newpage

\appendix
\section{Analysis of Variance and proof of Theorem \ref{thm:variance}}\label{app:variance}

In this Section we firstly report the bias-variance decomposition equation, that will be exploited in this section to prove the variance expression reported in the main paper, and the bias in the next section. 

Considering a task $\mathcal{T}_i$, the Mean Squared Error (MSE), can be decomposed into three terms (bias-variance decomposition~\citep{hastie2009}):

\begin{equation}\label{eq:BiasVarDec}
\begin{gathered}
	      \underbrace{\mathbb{E}_{\mathbf{x},y_i,\mathcal{S}_i}[(\mathcal{M}_{\mathcal{S}_i}(\mathbf{x})-y_i)^2]}_{\text{MSE}}
	      = \underbrace{\mathbb{E}_{\mathbf{x},\mathcal{S}_i}[(\mathcal{M}_{\mathcal{S}_i}(\mathbf{x})-\bar{\mathcal{M}}_i(\mathbf{x}))^2]}_{\text{variance}}
	     \\\qquad+\underbrace{\mathbb{E}_{\mathbf{x}}[(\bar{\mathcal{M}_i}(\mathbf{x})-\bar{y}_i(\mathbf{x}))^2]}_{\text{bias}}
	      +\underbrace{\mathbb{E}_{\mathbf{x},y_i}[(\bar{y}_i(\mathbf{x})-y_i)^2]}_{\text{noise}},
	\end{gathered}
\end{equation}
where $\mathbf{x},y_i$ are respectively features and the target of a test sample associated with task $\mathcal{T}_i$ and $\mathcal{S}_i$ is the related training set. $\mathcal{M}_{\mathcal{S}_i}(\cdot)$ is the model trained with $\mathcal{S}_i$, and $\bar{\mathcal{M}}_i(\cdot)$, $\bar{y}_i$, are its expected value \wrt\ $\mathcal{S}_i$ and the expected value of the test target $y_i$ \wrt\ the input features $\mathbf{x}$.

In this section, recalling the data generation process underlying the relationship between features and target of a generic task $i$ as: $y_i=f_i(x_1\dots x_D) + \epsilon_i,\ \epsilon_i\sim \mathcal{N}(0,\sigma_i^2)$, and the general linear regression model under analysis: $\hat{y}_i=\hat{\psi}_\iota=\widehat{\frac{1}{K_\iota}\sum_{k:y_k \in \mathcal{P}_\iota} y_k}=\hat{w}_1^\iota \phi_1 + \dots + \hat{w}_d^\iota \phi_d$, we will firstly focus on a simplified version and subsequently extend the analysis to this more general case. 

\subsection{Linear model of the original features, single-task}

Firstly we estimate the output $y_i$ with a multivariate linear regression on the $D$ original features. Each model is predicted as:
\begin{equation}\label{eq:singleModelOriginalFeatures}
    \hat{y}_i=\hat{w}_1^ix_1 + \dots + \hat{w}_D^ix_D.
\end{equation}
To compute the variance of the linear model, we need the variance and the expected value of the coefficients, which are reported in the following Lemma.

\begin{lemma}\label{lem:varExpCoeff}
    In the asymptotic case, the variance and the expected value (conditioned to the training set feature matrix $\mathbf{X}$) of the linear regression coefficients of the model in Equation \ref{eq:singleModelOriginalFeatures}, trained to estimate the $i$-th task, are respectively:
    \begin{align}
    \begin{split}
var_{\mathcal{T}}&(\mathbf{\hat{w}^i}\lvert\mathbf{X}) 
     = \frac{\sigma^2_i}{(n-1)}\mathcal{P_D} \\
     &= \frac{\sigma^2_i}{(n-1)} \begin{bmatrix}
    \frac{1}{\sigma^2_{x_1|x^{-1}}} & \frac{-\rho_{x_1,x_2|x^{-1,2}}}{\sigma_{x_1|x^{-1}}\cdot\sigma_{x_2|x^{-2}}} & \dots & \frac{-\rho_{x_1,x_D|x^{-1,D}}}{\sigma_{x_1|x^{-1}}\cdot\sigma_{x_D|x^{-D}}}\\
    \dots & \dots & \dots & \dots\\
    \frac{-\rho_{x_1,x_D|x^{-1,D}}}{\sigma_{x_1|x^{-1}}\cdot\sigma_{x_D|x^{-D}}} & \frac{-\rho_{x_2,x_D|x^{-2,D}}}{\sigma_{x_2|x^{-2}}\cdot\sigma_{x_D|x^{-D}}} & \dots & \frac{1}{\sigma^2_{x_D|x^{-D}}}
    \end{bmatrix},
    \end{split}
\end{align}
\begin{align}
\begin{split}
\E_{\mathcal{T}}[\mathbf{\hat{w}^i}\lvert\mathbf{X}] &= \frac{1}{(n-1)}\mathcal{P}_D\mathbf{X}^\intercal \mathbf{y}^i \\
&= \begin{bmatrix} 
    \frac{cov(x_1,f_i(x))}{\sigma^2_{x_1\lvert x^{-1}}} - \sum_{\substack{k=1 \\ k\neq 1}}^{D}\Big( \frac{\rho_{x_1,x_k\lvert x^{-1,k}}}{\sigma_{x_1\lvert x^{-1}}\cdot \sigma_{x_k\lvert x^{-k}}} \Big) cov(x_k,f_i(x))\\
    \dots \\
    \frac{cov(x_D,f_i(x))}{\sigma^2_{x_D\lvert x^{-D}}} - \sum_{\substack{k=1 \\ k\neq D}}^{D}\Big( \frac{\rho_{x_D,x_k\lvert x^{-D,k}}}{\sigma_{x_D\lvert x^{-D}}\cdot \sigma_{x_k\lvert x^{-k}}} \Big) cov(x_k,f_i(x))
\end{bmatrix},
\end{split}
\end{align}
    where $\mathcal{P_D}$ is the inverse of the covariance matrix, i.e., the precision matrix, $\sigma_{x_j|x^{-j}}$ is the partial variance of a feature $x_j$ given the others, and $\rho_{x_j,x_k|x^{-j,k}}$ is the partial correlation between the features $x_j,x_k$, given the others.
\end{lemma}
\begin{proof}
    Recalling that in general, for a linear regression model \citep{johnson2007}:

\begin{equation*}
\label{eq:VarEst}
var_{\mathcal{T}}(\hat{w}\lvert\mathbf{X}) = (\mathbf{X}^T\mathbf{X})^{-1}\sigma^2,
\end{equation*}

we compute the variance of the estimated coefficients conditioned on the training set:

\begin{align*}
var_{\mathcal{T}}(\mathbf{\hat{w}^i}\lvert\mathbf{X}) 
    &= \Bigg(\begin{bmatrix} x^1_1 & \dots & x^n_1 \\ 
    & \dots & \\
    x_D^1 & \dots & x_D^n \end{bmatrix} \begin{bmatrix} x^1_1 & \dots & x^1_D \\  & \dots & \\ x_1^n & \dots & x_D^n \end{bmatrix}\Bigg)^{-1}\sigma^2_i \\
    &= \Bigg(\begin{bmatrix}
    \sum_{i=1}^n(x_1^i)^2 & \sum_{i=1}^n(x_1^ix_2^i) & \dots & \sum_{i=1}^n(x_1^ix_D^i)\\
    \sum_{i=1}^n(x_1^ix_2^i) & \sum_{i=1}^n(x_2^i)^2 & \dots & \sum_{i=1}^n(x_2^ix_D^i)\\
    \dots & \dots & \dots & \dots\\
    \sum_{i=1}^n(x_1^ix_D^i) & \sum_{i=1}^n(x_2^ix_D^i) & \dots & \sum_{i=1}^n(x_D^i)^2
    \end{bmatrix}\Bigg)^{-1}\sigma^2_i \\
    &= \frac{\sigma^2_i}{(n-1)} \Bigg(\begin{bmatrix}
    \hat{\sigma}^2_{x_1} & \hat{cov}(x_1,x_2) & \dots & \hat{cov}(x_1,x_D)\\
    \dots & \dots & \dots & \dots\\
    \hat{cov}(x_1,x_D) & \hat{cov}(x_2,x_D) & \dots & \hat{\sigma}^2_{x_D}
    \end{bmatrix}\Bigg)^{-1}\\
    &= \frac{\sigma^2}{(n-1)}\mathcal{P_D}. 
\end{align*}

The last equality holds since the inverse of the covariance matrix of the $D$ original features is the precision matrix $\mathcal{P_D}$ (or concentration matrix), which can be rewritten as follows (directly considering the asymptotic case to substitute the estimators with the quantities they estimate):

\begin{equation*}
    \mathcal{P_D} = \begin{bmatrix}
    \frac{1}{\sigma^2_{x_1|x^{-1}}} & \frac{-\rho_{x_1,x_2|x^{-1,2}}}{\sigma_{x_1|x^{-1}}\cdot\sigma_{x_2|x^{-2}}} & \dots & \frac{-\rho_{x_1,x_D|x^{-1,D}}}{\sigma_{x_1|x^{-1}}\cdot\sigma_{x_D|x^{-D}}}\\
    \dots & \dots & \dots & \dots\\
    \frac{-\rho_{x_1,x_D|x^{-1,D}}}{\sigma_{x_1|x^{-1}}\cdot\sigma_{x_D|x^{-D}}} & \frac{-\rho_{x_2,x_D|x^{-2,D}}}{\sigma_{x_2|x^{-2}}\cdot\sigma_{x_D|x^{-D}}} & \dots & \frac{1}{\sigma^2_{x_D|x^{-D}}}
    \end{bmatrix}.
\end{equation*}
In the equation, the partial variance is the variance of the residual of the regression of the other features on the current one, while the partial correlation is the correlation of the residual of the linear regression of the other features on them.

This proves the expression of variance of the lemma. Similarly, with the expected values, we have:

\begin{align*}
\E_{\mathcal{T}}[\mathbf{\hat{w}^i}\lvert\mathbf{X}] & = (\mathbf{X}^\intercal\mathbf{X})^{-1}\mathbf{X}^\intercal \mathbf{y_i} = \frac{1}{(n-1)}\mathcal{P}_D\mathbf{X}^\intercal \mathbf{y_i} \\
& = \frac{1}{(n-1)} \begin{bmatrix} \sum_{j=1}^n \Big[\frac{x_1^j}{\sigma^2_{x_1\lvert x^{-1}}} - \sum_{\substack{k=1 \\ k\neq 1}}^{D} \frac{\rho_{x_1,x_k\lvert x^{-1,k}}\cdot x_k^j}{\sigma_{x_1\lvert x^{-1}}\cdot \sigma_{x_k\lvert x^{-k}}} \Big] \cdot f_i(x)^j \\
\dots \\
\sum_{j=1}^n \Big[\frac{x_D^j}{\sigma^2_{x_D\lvert x^{-D}}} - \sum_{\substack{k=1 \\ k\neq D}}^{D} \frac{\rho_{x_k,x_D\lvert x^{-k,D}}\cdot x_k^j}{\sigma_{x_k\lvert x^{-k}}\cdot \sigma_{x_D\lvert x^{-D}}} \Big] \cdot f_i(x)^j
\end{bmatrix}\\
&= \begin{bmatrix} 
    \frac{cov(x_1,y)}{\sigma^2_{x_1\lvert x^{-1}}} - \sum_{\substack{k=1 \\ k\neq 1}}^{D}\Big( \frac{\rho_{x_1,x_k\lvert x^{-1,k}}}{\sigma_{x_1\lvert x^{-1}}\cdot \sigma_{x_k\lvert x^{-k}}} \Big) cov(x_k,f_i(x))\\
    \dots \\
    \frac{cov(x_D,y)}{\sigma^2_{x_D\lvert x^{-D}}} - \sum_{\substack{k=1 \\ k\neq D}}^{D}\Big( \frac{\rho_{x_D,x_k\lvert x^{-D,k}}}{\sigma_{x_D\lvert x^{-D}}\cdot \sigma_{x_k\lvert x^{-k}}} \Big) cov(x_k,f_i(x))
\end{bmatrix}
\end{align*}

This concludes the proofs, showing the second equality reported in the lemma.

\end{proof}

\begin{remark}
In the case of $D$-dimensional linear data generation process, the expected value of the coefficients is the vector of the coefficients of the data generation process itself, and the model is unbiased, as expected.
\end{remark}

We are now ready to show the variance of the model in this single-task setting.

\begin{theorem}

    In the asymptotic case, the variance of the linear regression model of Equation \ref{eq:singleModelOriginalFeatures}, trained to estimate the $i$-th task, is related to the characteristics of the features and with the coefficients of the linear regression model as:
    \begin{align*}
        var_D^i&=\frac{\sigma^2_i}{(n-1)}\cdot \sum_{k=1}^{D} \Big\{ \frac{\sigma^2_{x_k}}{\sigma^2_{x_k\lvert x^{-k}}} - \sum_{j=1,j\neq k}^{D}\frac{cov(x_k,x_j)\cdot \rho_{x_k,x_j\lvert x^{-k,j}}}{\sigma_{x_k\lvert x^{-k}}\cdot \sigma_{x_j\lvert x^{-j}}} \Big\}\\
        &= \frac{\sigma^2_i}{(n-1)}\cdot \sum_{k=1}^{D}\E_{\mathcal{T}}[\hat{w}^k_k\lvert\mathbf{X}],
    \end{align*}

    where $\E_{\mathcal{T}}[\hat{w}^k_k\lvert\mathbf{X}]$ is the expected value of the coefficient associated with the variable $x_k$ of the regression of the full set of features (including the variable $x_k$ itself) on $x_k$, which is equal to $1$ (and the other coefficients equal to $0$) when the features are independent.
    
    Furthermore, it is equal to:

    \begin{align*}
    var_D^i=& \frac{\sigma^2_i}{(n-1)}\cdot D.
    \end{align*}
    
\end{theorem}

\begin{proof}
    By definition of variance of the model, in this setting it is defined as $var_D^i:=\E_{x}\E_\mathcal{T}[(\hat{w}_1^ix_1+\dots+\hat{w}_D^ix_D-\E_{\mathcal{T}}[\hat{w}^i_1x_1+\dots+\hat{w}^i_Dx_D])^2]$. Recalling the expression of variance of the coefficients of the previous lemma, in the asymptotic case we get:

\begin{align*}
    var_D^i=&\E_{x}\E_\mathcal{T}[(\hat{w}_1^ix_1+\dots+\hat{w}_D^ix_D-\E_{\mathcal{T}}[\hat{w}^i_1x_1+\dots+\hat{w}^i_Dx_D])^2]\\
    =& \sum_{k=1}^{D}\sigma^2_{x_k} \cdot var_\mathcal{T}(\hat{w}^i_k) + 2\sum_{k=1}^{D-1}\sum_{j=k+1}^D cov_x(x_k,x_j)\cdot cov_\mathcal{T}(\hat{w}_k^i,\hat{w}_j^i)\\
    =& \frac{\sigma^2_i}{(n-1)}\cdot\Big\{ \sum_{k=1}^D\frac{\sigma^2_{x_k}}{\sigma^2_{x_k\lvert x^{-k}}} - 2 \sum_{k=1}^{D-1}\sum_{j=k+1}^D \frac{cov(x_k,x_j)\cdot \rho_{x_k,x_j\lvert x^{-k,j}}}{\sigma_{x_k\lvert x^{-k}}\cdot \sigma_{x_j\lvert x^{-j}}} \Big\}.
\end{align*}

Firstly, exploiting the expression of expected values of coefficients adapted from the previous lemma, we prove the first expression of the theorem:

\begin{align*}
    var_D^i=& \frac{\sigma^2_i}{(n-1)}\cdot\Big\{ \sum_{k=1}^D\frac{\sigma^2_{x_k}}{\sigma^2_{x_k\lvert x^{-k}}} - 2 \sum_{k=1}^{D-1}\sum_{j=k+1}^D \frac{cov(x_k,x_j)\cdot \rho_{x_k,x_j\lvert x^{-k,j}}}{\sigma_{x_k\lvert x^{-k}}\cdot \sigma_{x_j\lvert x^{-j}}} \Big\}\\
    =& \frac{\sigma^2_i}{(n-1)}\cdot \sum_{k=1}^{D} \Big\{ \frac{\sigma^2_{x_k}}{\sigma^2_{x_k\lvert x^{-k}}} - \sum_{j=1,j\neq k}^{D}\frac{cov(x_k,x_j)\cdot \rho_{x_k,x_j\lvert x^{-k,j}}}{\sigma_{x_k\lvert x^{-k}}\cdot \sigma_{x_j\lvert x^{-j}}} \Big\}\\
    &= \frac{\sigma^2_i}{(n-1)}\cdot \sum_{k=1}^{D}\E_{\mathcal{T}}[\hat{w}^k_k\lvert\mathbf{X}].
\end{align*}

For the second result, recalling the expression of the variance, The following considerations hold:

\begin{align*}
    var_D^i=& \sum_{k=1}^{D}\sum_{j=1}^D cov_x(x_k,x_j)\cdot cov_\mathcal{T}(\hat{w}_k^i,\hat{w}_j^i).
\end{align*}
This is the sum of $D\times D$ elements, basically the sum of the elements of the Hadamard product between the elements of the covariance matrix of the coefficient and the covariance matrix of the features. Recalling that $var_{\mathcal{T}}(\hat{w}^i_D\lvert\mathbf{X}) 
     = \frac{\sigma^2_i}{(n-1)}(\mathbf{X}^\intercal\mathbf{X})^{-1}$ and $cov(\mathbf{x})=\mathbf{X}^\intercal\mathbf{X}$, the following equivalences hold, proving the result:

\begin{align*}
    &var_D^i = \sum_{k=1}^{D}\sum_{j=1}^D cov_x(x_k,x_j)\cdot cov_\mathcal{T}(\hat{w}_k^i,\hat{w}_j^i)=\\
    & = \frac{\sigma^2_i}{(n-1)} \begin{bmatrix}
        1 & \dots & 1
    \end{bmatrix} (\mathbf{X}^\intercal\mathbf{X})^{-1}(\mathbf{X}^\intercal\mathbf{X}) \begin{bmatrix}
        1 \\ \dots \\ 1
    \end{bmatrix}  
    =\frac{\sigma^2_i}{(n-1)} D .
\end{align*}
\end{proof}

\begin{remark}
    The result can be trivially confirmed in the settings that consider two and three input features, where the inverse $(\mathbf{X}^\intercal\mathbf{X})^{-1}$ can be computed explicitly, exploiting the closed form of the inverses of $2\times2,3\times3$ matrices.
\end{remark}

\subsection{Extension to transformed features and aggregated targets}
In this section, we elaborate on the results of the previous subsection. 

Firstly, given the original targets $y_1\dots y_L$, we apply a multivariate linear regression after having aggregated clusters of targets, resulting in $l$ aggregated targets $\psi_1\dots\psi_l$. We still consider the $D$ original features. Focusing again on the $i$-th task with target $y_i$ and assuming that it is associated with the $\iota$ cluster $\psi_\iota = \frac{1}{K}\sum_{k:y_k \in \mathcal{P}_\iota} y_k$, we estimate the target $y_i$ with:
$$\hat{y}_i=\hat{\psi}_\iota=\widehat{\frac{1}{K_\iota}\sum_{k:y_k \in \mathcal{P}_\iota} y_k}=\hat{w}_1^\iota x_1 + \dots + \hat{w}_D^\iota x_D.$$

Defining $\bar{\sigma}_\iota^2:=var(\frac{1}{K_\iota}\sum_{k:y_k \in \mathcal{P}_\iota} \epsilon_k)$, recalling that we have the same features of the previous subsection, we can write the variance of the coefficients as:

\begin{align*}
var_{\mathcal{T}}(\hat{w}^\iota_D\lvert\mathbf{X}) 
     &= \frac{\bar{\sigma}_\iota^2}{(n-1)}\mathcal{P_D} \\
     &= \frac{\bar{\sigma}_\iota^2}{(n-1)} \begin{bmatrix}
    \frac{1}{\sigma^2_{x_1|x^{-1}}} & \frac{-\rho_{x_1,x_2|x^{-1,2}}}{\sigma_{x_1|x^{-1}}\cdot\sigma_{x_2|x^{-2}}} & \dots & \frac{-\rho_{x_1,x_D|x^{-1,D}}}{\sigma_{x_1|x^{-1}}\cdot\sigma_{x_D|x^{-D}}}\\
    \dots & \dots & \dots & \dots\\
    \frac{-\rho_{x_1,x_D|x^{-1,D}}}{\sigma_{x_1|x^{-1}}\cdot\sigma_{x_D|x^{-D}}} & \frac{-\rho_{x_2,x_D|x^{-2,D}}}{\sigma_{x_2|x^{-2}}\cdot\sigma_{x_D|x^{-D}}} & \dots & \frac{1}{\sigma^2_{x_D|x^{-D}}}
    \end{bmatrix}. 
\end{align*}

Then, we estimate each task with a single-task multivariate linear regression on the $d$ reduced features $\phi_1\dots\phi_d$. The $i$-th model is therefore predicted as:
$$\hat{y}_i=\hat{w}_1^i\phi_1 + \dots + \hat{w}_d^i\phi_d.$$

The variance of the coefficients, therefore, becomes:

\begin{align*}
var_{\mathcal{T}}(\hat{w}^i_d\lvert\mathbf{X}) 
     &= \frac{\sigma^2_i}{(n-1)}\mathcal{P}_d(\phi) \\
     &= \frac{\sigma^2_i}{(n-1)} \begin{bmatrix}
    \frac{1}{\sigma^2_{\phi_1|\phi^{-1}}} & \frac{-\rho_{\phi_1,\phi_2|\phi^{-1,2}}}{\sigma_{\phi_1|\phi^{-1}}\cdot\sigma_{\phi_2|\phi^{-2}}} & \dots & \frac{-\rho_{\phi_1,\phi_D|\phi^{-1,d}}}{\sigma_{\phi_1|\phi^{-1}}\cdot\sigma_{\phi_d|\phi^{-d}}}\\
    \dots & \dots & \dots & \dots\\
    \frac{-\rho_{\phi_1,\phi_d|\phi^{-1,d}}}{\sigma_{\phi_1|\phi^{-1}}\cdot\sigma_{\phi_d|\phi^{-d}}} & \frac{-\rho_{\phi_2,\phi_d|\phi^{-2,d}}}{\sigma_{\phi_2|\phi^{-2}}\cdot\sigma_{\phi_d|\phi^{-d}}} & \dots & \frac{1}{\sigma^2_{\phi_d|\phi^{-d}}}
    \end{bmatrix}. 
\end{align*}

Finally, combining both aggregations of targets and features, we get the following variance of the coefficients:

\begin{align*}
var_{\mathcal{T}}(\hat{w}^\iota_d\lvert\mathbf{X}) 
     &= \frac{\bar{\sigma}_\iota^2}{(n-1)}\mathcal{P}_d(\phi) \\
     &= \frac{\bar{\sigma}_\iota^2}{(n-1)} \begin{bmatrix}
    \frac{1}{\sigma^2_{\phi_1|\phi^{-1}}} & \frac{-\rho_{\phi_1,\phi_2|\phi^{-1,2}}}{\sigma_{\phi_1|\phi^{-1}}\cdot\sigma_{\phi_2|\phi^{-2}}} & \dots & \frac{-\rho_{\phi_1,\phi_D|\phi^{-1,d}}}{\sigma_{\phi_1|\phi^{-1}}\cdot\sigma_{\phi_d|\phi^{-d}}}\\
    \dots & \dots & \dots & \dots\\
    \frac{-\rho_{\phi_1,\phi_d|\phi^{-1,d}}}{\sigma_{\phi_1|\phi^{-1}}\cdot\sigma_{\phi_d|\phi^{-d}}} & \frac{-\rho_{\phi_2,\phi_d|\phi^{-2,d}}}{\sigma_{\phi_2|\phi^{-2}}\cdot\sigma_{\phi_d|\phi^{-d}}} & \dots & \frac{1}{\sigma^2_{\phi_d|\phi^{-d}}}
    \end{bmatrix}. 
\end{align*}

Therefore, we can exploit these results to obtain the expression of the variance in the three configurations. Note that the last configuration is the most general. It contains all the previous ones as particular cases, and it \textbf{proves Theorem \ref{thm:variance} of the main paper}.

\begin{theorem}
    In the asymptotic case, estimating the $i$-th task with the average $\psi_\iota$ of a set of tasks that contains $y_i$, the variance of the model is:
    
    \begin{align*}
    var_D^\iota=& \frac{\bar{\sigma}_\iota^2}{(n-1)}\cdot D.
\end{align*}

On the other hand, estimating the $i$-th task in a single-task setting, with a set of transformed features $\phi_1\dots\phi_d$, the variance of the model is:

\begin{equation*}
    var_d^i = \frac{\sigma^2_i}{(n-1)}\cdot d.
\end{equation*}

Finally, combining the aggregation of targets and the transformation of features, we get the expression reported in Theorem \ref{thm:variance} in the main paper:

\begin{equation*}
    var_d^\iota = \frac{\bar{\sigma}_\iota^2}{(n-1)}\cdot d.
\end{equation*}
\end{theorem}

\begin{proof}
Exploiting the expression of variance of the coefficients, with similar arguments of the previous subsection, in the first case, we get:
\begin{align*}
    &var_D^\iota = \sum_{k=1}^{D}\sum_{j=1}^D cov_x(x_k,x_j)\cdot cov_\mathcal{T}(\hat{w}_k^\iota,\hat{w}_j^\iota)=\\
    & = \frac{\sigma^2_\iota}{(n-1)} \begin{bmatrix}
        1 & \dots & 1
    \end{bmatrix} (\mathbf{X}^\intercal\mathbf{X})^{-1}(\mathbf{X}^\intercal\mathbf{X}) \begin{bmatrix}
        1 \\ \dots \\ 1
    \end{bmatrix}  
    =\frac{\sigma^2_i}{(n-1)} D .
\end{align*}

In the second case we get:
\begin{align*}
    &var_d^i = \sum_{k=1}^{d}\sum_{j=1}^d cov_x(\phi_k,\phi_j)\cdot cov_\mathcal{T}(\hat{w}_k^i,\hat{w}_j^i)=\\
    & = \frac{\sigma^2_i}{(n-1)} \begin{bmatrix}
        1 & \dots & 1
    \end{bmatrix} (\mathbf{\Phi}^\intercal\mathbf{\Phi})^{-1}(\mathbf{\Phi}^\intercal\mathbf{\Phi}) \begin{bmatrix}
        1 \\ \dots \\ 1
    \end{bmatrix}  
    =\frac{\sigma^2_i}{(n-1)} d .
\end{align*}

Finally, by combining the two cases, we get:

\begin{align*}
    &var_d^\iota = \sum_{k=1}^{d}\sum_{j=1}^d cov_x(\phi_k,\phi_j)\cdot cov_\mathcal{T}(\hat{w}_k^\iota,\hat{w}_j^\iota)=\\
    & = \frac{\sigma^2_\iota}{(n-1)} \begin{bmatrix}
        1 & \dots & 1
    \end{bmatrix} (\mathbf{\Phi}^\intercal\mathbf{\Phi})^{-1}(\mathbf{\Phi}^\intercal\mathbf{\Phi}) \begin{bmatrix}
        1 \\ \dots \\ 1
    \end{bmatrix}  
    =\frac{\sigma^2_\iota}{(n-1)} d .
\end{align*}
\end{proof}

\section{Analysis of Bias and proof of Theorem \ref{thm:bias}}\label{app:bias}
As in the previous section, to show the bias of the model, we will first focus on a single-task linear regression of the original features, subsequently extending to the general case of an aggregated target $\phi_\iota$ and a reduced set of inputs $\phi_1,\dots,\phi_d$.

\subsection{Linear model of the original features, single-task}

Firstly we estimate the output $y_i$ with a multivariate linear regression on the $D$ original features. Each model is predicted as shown in Equation \ref{eq:singleModelOriginalFeatures}. To compute the bias of this linear model, we will exploit the expression of the expected value of the coefficients that can be found in Lemma \ref{lem:varExpCoeff}. Additionally, we will exploit the following Lemma, that shows the expression of the coefficient o multiple correlation in a linear regression setting.

\begin{lemma}
    Considering the linear regression on the target $y_i$ of the features $x_1,\dots,x_D$, the coefficient of multiple correlation $R^2_{D,i}$ is:
    \begin{align*}
        R^2_{D,i} &= \sum_{k=1}^D \Big\{\frac{\sigma^2_{x_k}}{\sigma^2_{x_k\lvert x^{-k}}}\rho^2_{x_k,y_i} - \rho_{x_k,y_i}\sum_{\substack{h=1 \\ h\neq k}}^{D}\Big( \frac{\rho_{x_h,x_k\lvert x^{-h,k}}}{\sigma_{x_h\lvert x^{-h}}\cdot \sigma_{x_k\lvert x^{-k}}} \Big) \rho_{x_h,y_i}\sigma_{x_h}\sigma_{x_k}\Big\}\\
        & = \frac{1}{\sigma^2_{y_i}} \sum_{k=1}^D \Big\{\frac{cov(x_k,y_i)^2}{\sigma^2_{x_k\lvert x^{-k}}} - cov(x_k,y_i)\sum_{\substack{h=1 \\ h\neq k}}^{D}\Big( \frac{\rho_{x_h,x_k\lvert x^{-h,k}}}{\sigma_{x_h\lvert x^{-h}}\cdot \sigma_{x_k\lvert x^{-k}}} \Big) cov(x_h,y_i)\Big\}.
    \end{align*}
\end{lemma}

\begin{proof}
    Considering the definition of coefficient of determination (i.e., squared coefficient of multiple correlation):
    \begin{align*}
        R^2_{D,i} &= \begin{bmatrix}
            \rho_{x_1,y_i} & \dots & \rho_{x_D,y_i} 
        \end{bmatrix}
        corr(\mathbf{X})^{-1}\begin{bmatrix}
            \rho_{x_1,y_i} \\ \dots \\ \rho_{x_D,y_i}
        \end{bmatrix},
    \end{align*}
we need to compute the inverse of the correlation matrix of the features.

Given the inverse of the covariance matrix:

\begin{equation*}
    cov(\mathbf{X})^{-1} = \begin{bmatrix}
        \frac{1}{\sigma^2_{x_1}|x^{-1}} & \frac{\rho_{x_1,x_2\lvert x^{-1,2}}}{\sigma_{x_1\lvert x^{-1}}\cdot \sigma_{x_2\lvert x^{-2}}} & \dots & \frac{\rho_{x_1,x_D\lvert x^{-1,D}}}{\sigma_{x_1\lvert x^{-1}}\cdot \sigma_{x_D\lvert x^{-D}}}\\
        \dots & \dots & \dots & \dots \\
        \frac{\rho_{x_1,x_D\lvert x^{-1,D}}}{\sigma_{x_1\lvert x^{-1}}\cdot \sigma_{x_D\lvert x^{-D}}}  & \frac{\rho_{x_D,x_2\lvert x^{-D,2}}}{\sigma_{x_D\lvert x^{-D}}\cdot \sigma_{x_2\lvert x^{-2}}} & \dots & \frac{1}{\sigma^2_{x_D}|x^{-D}}\\
    \end{bmatrix}
\end{equation*}

decomposing it as follows:

\begin{equation*}
        cov(\mathbf{x}) = \begin{bmatrix}
            \sigma_1 & 0 & \dots & 0\\
            \dots & \dots & \dots & \dots \\
            0 & \dots & 0 & \sigma_D
        \end{bmatrix} 
        \begin{bmatrix}
            1 & \rho_{x_1,x_2} & \dots & \rho_{x_1,x_D}\\
            \dots & \dots & \dots & \dots \\
            \rho_{x_1,x_D} & \rho_{x_2,x_D} & \dots & 1
        \end{bmatrix}
        \begin{bmatrix}
            \sigma_1 & 0 & \dots & 0\\
            \dots & \dots & \dots & \dots \\
            0 & \dots & 0 & \sigma_D
        \end{bmatrix},
\end{equation*}
and recalling that the inverse of a product is the product of the inverses $(ABA)^{-1} = A^{-1}B^{-1}A^{-1}$, we get:

\begin{align*}
        &corr(\mathbf{X})^{-1} = \begin{bmatrix}
            \sigma_1 & 0 & \dots & 0\\
            \dots & \dots & \dots & \dots \\
            0 & \dots & 0 & \sigma_D
        \end{bmatrix}
        cov(\mathbf{X})^{-1}
        \begin{bmatrix}
            \sigma_1 & 0 & \dots & 0\\
            \dots & \dots & \dots & \dots \\
            0 & \dots & 0 & \sigma_D
        \end{bmatrix}
        \\
        &= \begin{bmatrix}
            \sigma_1 & 0 & \dots & 0\\
            \dots & \dots & \dots & \dots \\
            0 & \dots & 0 & \sigma_D
        \end{bmatrix}
        \begin{bmatrix}
        \frac{1}{\sigma^2_{x_1}|x^{-1}}  & \dots & \frac{\rho_{x_1,x_D\lvert x^{-1,D}}}{\sigma_{x_1\lvert x^{-1}}\cdot \sigma_{x_D\lvert x^{-D}}}\\
        \dots & \dots & \dots \\
        \frac{\rho_{x_1,x_D\lvert x^{-1,D}}}{\sigma_{x_1\lvert x^{-1}}\cdot \sigma_{x_D\lvert x^{-D}}}  & \dots & \frac{1}{\sigma^2_{x_D}|x^{-D}}\\
    \end{bmatrix}
        \begin{bmatrix}
            \sigma_1 & 0 & \dots & 0\\
            \dots & \dots & \dots & \dots \\
            0 & \dots & 0 & \sigma_D
        \end{bmatrix}\\
        &= \begin{bmatrix}
        \frac{\sigma^2_{x_1}}{\sigma^2_{x_1}|x^{-1}} & \frac{\sigma_{x_1}\sigma_{x_2}\rho_{x_1,x_2\lvert x^{-1,2}}}{\sigma_{x_1\lvert x^{-1}}\cdot \sigma_{x_2\lvert x^{-2}}} & \dots & \frac{\sigma_{x_1}\sigma_{x_D}\rho_{x_1,x_D\lvert x^{-1,D}}}{\sigma_{x_1\lvert x^{-1}}\cdot \sigma_{x_D\lvert x^{-D}}}\\
        \dots & \dots & \dots & \dots \\
        \frac{\sigma_{x_1}\sigma_{x_D}\rho_{x_1,x_D\lvert x^{-1,D}}}{\sigma_{x_1\lvert x^{-1}}\cdot \sigma_{x_D\lvert x^{-D}}}  & \frac{\sigma_{x_D}\sigma_{x_2}\rho_{x_D,x_2\lvert x^{-D,2}}}{\sigma_{x_D\lvert x^{-D}}\cdot \sigma_{x_2\lvert x^{-2}}} & \dots & \frac{\sigma_{x_D}^2}{\sigma^2_{x_D}|x^{-D}}\\
    \end{bmatrix}.
\end{align*}

Therefore:

\begin{align*}
        & R^2_{D,i} = \begin{bmatrix}
            \rho_{x_1,y_i} & \dots & \rho_{x_D,y_i} 
        \end{bmatrix}
        \begin{bmatrix}
        \frac{\sigma^2_{x_1}}{\sigma^2_{x_1}|x^{-1}}  & \dots & \frac{\sigma_{x_1}\sigma_{x_D}\rho_{x_1,x_D\lvert x^{-1,D}}}{\sigma_{x_1\lvert x^{-1}}\cdot \sigma_{x_D\lvert x^{-D}}}\\
        \dots & \dots & \dots \\
        \frac{\sigma_{x_1}\sigma_{x_D}\rho_{x_1,x_D\lvert x^{-1,D}}}{\sigma_{x_1\lvert x^{-1}}\cdot \sigma_{x_D\lvert x^{-D}}}  & \dots & \frac{\sigma_{x_D}^2}{\sigma^2_{x_D}|x^{-D}}\\
    \end{bmatrix}\begin{bmatrix}
            \rho_{x_1,y_i} \\ \dots \\ \rho_{x_D,y_i}
        \end{bmatrix}\\
    & = \begin{bmatrix}
        {\scriptstyle\frac{\sigma^2_{x_1}}{\sigma^2_{x_1\lvert x^{-1}}}\rho_{x_1,y_i} - \sum_{\substack{h=1 \\ h\neq 1}}^{D}\Big( \frac{\rho_{x_h,x_1\lvert x^{-h,1}}}{\sigma_{x_h\lvert x^{-h}}\cdot \sigma_{x_1\lvert x^{-1}}} \Big) \rho_{x_h,y_i}\sigma_{x_h}\sigma_{x_1}} & \dots 
    \end{bmatrix}
    \begin{bmatrix}
            \rho_{x_1,y_i} \\ \dots \\ \rho_{x_D,y_i}
        \end{bmatrix}\\
    & = \sum_{k=1}^D \Big\{\frac{\sigma^2_{x_k}}{\sigma^2_{x_k\lvert x^{-k}}}\rho^2_{x_k,y_i} - \rho_{x_k,y_i}\sum_{\substack{h=1 \\ h\neq k}}^{D}\Big( \frac{\rho_{x_h,x_k\lvert x^{-h,k}}}{\sigma_{x_h\lvert x^{-h}}\cdot \sigma_{x_k\lvert x^{-k}}} \Big) \rho_{x_h,y_i}\sigma_{x_h}\sigma_{x_k}\Big\}.
\end{align*}
\end{proof}

\begin{remark}
    The same result of the Lemma can be computed by considering the definition of the coefficient of determination in terms of the residual and total sum of squares:
    \begin{equation*}
        R^2_{D,i} = \frac{SS_{reg}}{SS_{tot}} = \frac{\mathbf{\hat{w}}^\intercal X^\intercal \mathbf{y}_i - \frac{1}{N} (\mathbf{1}^\intercal \mathbf{y}_i)^2}{\mathbf{y}_i^\intercal\mathbf{y}_i - \frac{1}{N} (\mathbf{1}^\intercal \mathbf{y}_i)^2}.
    \end{equation*}
\end{remark}

The following theorem finally shows the expression of bias of the model in this linear single-task setting.

\begin{theorem}
        In the asymptotic case, the bias of the linear regression model of Equation \ref{eq:singleModelOriginalFeatures}, trained to estimate the $i$-th task, is related to the characteristics of the features and with the coefficients of the linear regression model as:
    \begin{align*}
        bias_D^i&=var_x(\sum_{k=1}^D\E_{\mathcal{T}}[\hat{w}_k^i]x_k) + \sigma^2_{f_i} - 2 cov(\sum_{k=1}^D\E_{\mathcal{T}}[\hat{w}_k^i]x_k,f_i(x)),
    \end{align*}

    where the expected values of the regression coefficients $\E_{\mathcal{T}}[\hat{w}_k^i]$ can be computed with the results of Lemma \ref{lem:varExpCoeff}.
    
    Furthermore, it is equal to:

    \begin{align*}
    bias_{D}^i = \sigma^2_{f_i} (1-R^2_{D,i}).
    \end{align*}
\end{theorem}

\begin{proof}
    The first expression can be obtained by definition of bias of the model and exploiting the hypothesis of null expected values of $f_i$ and $x_1,\dots,x_D$, as:
    \begin{align*}
    bias_D^i&=\E_{x}[(\sum_{k=1}^D\E_{\mathcal{T}}[\hat{w}_k^i]x_k-f_i(x))^2] = var_x(\sum_{k=1}^D\E_{\mathcal{T}}[\hat{w}_k^i]x_k-f_i(x))\\
    &= var_x(\sum_{k=1}^D\E_{\mathcal{T}}[\hat{w}_k^i]x_k) + var_x(f_i(x)) - 2 cov(\sum_{k=1}^D\E_{\mathcal{T}}[\hat{w}_k^i]x_k,f_i(x)).
\end{align*}

This already shows that the bias depends on the variance of the predicted and the real model, decreased by how much the real and the predicted model are related (through covariance). However, by exploiting the coefficient of multiple correlation and the partial correlation, we can refine this expression to estimate it directly from the data. 

Exploiting the expression of the expected value of the coefficients, the last term of the equation above can be written as:

\begin{align*}
    &- 2 cov(\sum_{k=1}^D\E_{\mathcal{T}}[\hat{w}_k^i]x_k,f_i(x)) = -2 \sum_{k=1}^Dcov(x_k,f_i)\E_{\mathcal{T}}[\hat{w}_k^i]\\
    & = -2 \sum_{k=1}^D \Big\{\frac{cov(x_k,f_i)^2}{\sigma^2_{x_k\lvert x^{-k}}} - \sum_{\substack{h=1 \\ h\neq k}}^{D}\Big( \frac{\rho_{x_h,x_k\lvert x^{-h,k}}}{\sigma_{x_h\lvert x^{-h}}\cdot \sigma_{x_k\lvert x^{-k}}} \Big) cov(x_h,f_i)cov(x_k,f_i)\Big\} \\
    &= -2\sigma^2_{f_i}R^2_{D,i}.
\end{align*}

The bias is, therefore, equal to:
\begin{align*}
    bias_{D}^i = var_x(\sum_{k=1}^D\E_{\mathcal{T}}[\hat{w}_k^i]x_k) + \sigma^2_{f_i} -2\sigma^2_{f_i}R^2_{D,i}.
\end{align*}

We finally characterize the first term, i.e. the variance of the linear estimator, concluding the proof of the theorem. 

\begin{lemma}
    The following equality holds:
\begin{equation*}
    \sigma^2_{f_i}R^2_{D,i} = var_x(\sum_{k=1}^D\E_{\mathcal{T}}[\hat{w}_k^i]x_k).
\end{equation*}
\end{lemma}
\begin{proof}
Exploiting the properties of variance:
\begin{align*}
    var_x(\sum_{k=1}^D\E_{\mathcal{T}}[\hat{w}_k^i]x_k) &= \sum_{k=1}^D\sum_{j=1}^D cov(x_k,x_j) \E_{\mathcal{T}}[\hat{w}_k^i]\E_{\mathcal{T}}[\hat{w}_j^i]\\
    &= \sum_{k=1}^D\E_{\mathcal{T}}[\hat{w}_k^i]\Big(\sum_{j=1}^D cov(x_k,x_j) \E_{\mathcal{T}}[\hat{w}_j^i]\Big),
\end{align*}
and recalling that, from above:
\begin{align*}
    \sigma^2_{f_i}R^2_{D,i} = \sum_{k=1}^Dcov(x_k,f_i)\E_{\mathcal{T}}[\hat{w}_k^i],
\end{align*}

we need to show that, for each $k\in \{1,\dots,D\}$, the following equality holds:
\begin{equation*}
    \sum_{j=1}^D cov(x_k,x_j) \E_{\mathcal{T}}[\hat{w}_j^i] = cov(x_k,f_i).
\end{equation*}

This equality can be shown by exploiting the definition of partial covariance:

\begin{equation*}
    0 = cov(x_k,f_i|\mathbf{X}) = cov(x_k,f_i) - cov(x_k,\mathbf{X})cov(\mathbf{X})^{-1}cov(\mathbf{X},f_i)
\end{equation*}

where the first equality holds because the residual of the regression of $\mathbf{X}$ on $x_k$ is zero since $x_k$ is among the regressors, and the second equality holds by definition of partial covariance. Rewriting the equality that we want to show as: 

\begin{align*}
    cov(x_k,f_i) = \sum_{j=1}^D cov(x_k,x_j) \E_{\mathcal{T}}[\hat{w}_j^i] = cov(x_k,\mathbf{X}) \E_{\mathcal{T}}[\hat{\mathbf{w}}^i] \\=cov(x_k,\mathbf{X}) (\mathbf{X}^\intercal \mathbf{X})^{-1} \mathbf{X}^\intercal \mathbf{f_i} = cov(x_k,\mathbf{X}) cov(\mathbf{X})^{-1} cov(\mathbf{X},f_i),
\end{align*}
proves the statement of the lemma and concludes the proof of the theorem.
\end{proof}
\end{proof}

\begin{remark}
    The result is confirmed in the two- and three-dimensional case and in the case of $D$ mutually independent features, where all the quantities can be explicitly computed.
\end{remark}

\subsection{Extension to transformed features and aggregated targets}
As done for the variance, in this section, we elaborate on the results of the previous subsection and extend them. 

Firstly, given the original targets $y_1\dots y_L$, we apply a multivariate linear regression after having aggregated clusters of targets, resulting in $l$ aggregated targets $\psi_1\dots\psi_l$. We still consider the $D$ original features. Focusing again on the $i$-th task, defining $\bar{f}_\iota(x) := \frac{1}{K}\sum_{k:y_k \in \mathcal{P}_\iota} f_k(x)$, the expected values of coefficients are:

\begin{align*}
\E_{\mathcal{T}}[\hat{w}^\iota_D\lvert\mathbf{X}] = 
\begin{bmatrix}
    \frac{cov(x_1,\bar{f}_\iota(x))}{\sigma^2_{x_1\lvert x^{-1}}} - \sum_{\substack{k=1 \\ k\neq 1}}^{D}\Big( \frac{\rho_{x_1,x_k\lvert x^{-1,k}}}{\sigma_{x_1\lvert x^{-1}}\cdot \sigma_{x_k\lvert x^{-k}}} \Big) cov(x_k,\bar{f}_\iota(x))\\
    \dots \\
    \frac{cov(x_D,\bar{f}_\iota(x))}{\sigma^2_{x_D\lvert x^{-D}}} - \sum_{\substack{k=1 \\ k\neq D}}^{D}\Big( \frac{\rho_{x_D,x_k\lvert x^{-D,k}}}{\sigma_{x_D\lvert x^{-D}}\cdot \sigma_{x_k\lvert x^{-k}}} \Big) cov(x_k,\bar{f}_\iota(x))
\end{bmatrix}.
\end{align*}

\begin{theorem}\label{thm:generalBias}
    In the asymptotic case, estimating the $i$-th task with the average $\psi_\iota$ of a set of tasks that contains $y_i$, the bias of the model is:
    \begin{equation}\label{eq:biasAll}
\begin{split}
    bias_d^\iota &= \sigma^2_{f_i} + \sigma^2_{\psi_\iota}R^2_{d,\iota}  -2 ( cov(f_i,\psi_\iota) -  cov(f_i,\psi_\iota | \mathbf{\Phi}) )\\
    & = \sigma^2_{f_i} - \sigma^2_{\psi_\iota}R^2_{d,\iota}  +2 (cov(\psi_\iota,f_i - \psi_\iota | \mathbf{\Phi}) - cov(\psi_\iota,f_i - \psi_\iota)).
\end{split}
\end{equation}
\end{theorem}

\begin{proof}
    In this case, the bias w.r.t. the $i$-th task is:
\begin{align*}
    \E_{x}&[(\sum_{k=1}^D\E_{\mathcal{T}}[\hat{w}_k^\iota]x_k-f_i(x))^2] = var_x(\sum_{k=1}^D\E_{\mathcal{T}}[\hat{w}_k^\iota]x_k-f_i(x))\\
    &= var_x(\sum_{k=1}^D\E_{\mathcal{T}}[\hat{w}_k^\iota]x_k) + var_x(f_i(x)) - 2 cov(\sum_{k=1}^D\E_{\mathcal{T}}[\hat{w}_k^\iota]x_k,f_i(x)).
\end{align*}

Considering the linear regression on the target $\psi_\iota$ of the features $x_1,\dots,x_D$ the coefficient of multiple correlation $R^2_{D,\iota}$ is:
    \begin{align*}
        R^2_{D,\iota} &= \sum_{k=1}^D \Big\{\frac{\sigma^2_{x_k}}{\sigma^2_{x_k\lvert x^{-k}}}\rho^2_{x_k,\psi_\iota} - \rho_{x_k,\psi_\iota}\sum_{\substack{h=1 \\ h\neq k}}^{D}\Big( \frac{\rho_{x_h,x_k\lvert x^{-h,k}}}{\sigma_{x_h\lvert x^{-h}}\cdot \sigma_{x_k\lvert x^{-k}}} \Big) \rho_{x_h,\psi_\iota}\sigma_{x_h}\sigma_{x_k}\Big\}\\
        & = \frac{1}{\sigma^2_{\psi_\iota}} \sum_{k=1}^D \Big\{\frac{cov(x_k,\psi_\iota)^2}{\sigma^2_{x_k\lvert x^{-k}}} - cov(x_k,\psi_\iota)\sum_{\substack{h=1 \\ h\neq k}}^{D}\Big( \frac{\rho_{x_h,x_k\lvert x^{-h,k}}}{\sigma_{x_h\lvert x^{-h}}\cdot \sigma_{x_k\lvert x^{-k}}} \Big) cov(x_h,\psi_\iota)\Big\}.
    \end{align*}

Recalling the expression of the expected value of each coefficient:

\begin{equation*}
\E_{\mathcal{T}}[\hat{w}^\iota_k\lvert\mathbf{X}] = 
    \frac{cov(x_k,\psi_\iota)}{\sigma^2_{x_k\lvert x^{-k}}} - \sum_{\substack{h=1 \\ h\neq k}}^{D}\Big( \frac{\rho_{x_h,x_k\lvert x^{-h,k}}}{\sigma_{x_h\lvert x^{-h}}\cdot \sigma_{x_k\lvert x^{-k}}} \Big) cov(x_h,\psi_\iota),
\end{equation*}

the bias can be rewritten as:

\begin{equation*}
    bias_D^\iota = \sigma^2_{\psi_\iota}R^2_{D,\iota} + \sigma^2_{f_i} - 2 cov(\sum_{k=1}^D\E_{\mathcal{T}}[\hat{w}_k^\iota]x_k,f_i(x)).
\end{equation*}

It remains to explicit the last term: 
\begin{align*}
    &- 2 cov(\sum_{k=1}^D\E_{\mathcal{T}}[\hat{w}_k^\iota]x_k,f_i(x)) = -2 \sum_{k=1}^Dcov(x_k,f_i)\E_{\mathcal{T}}[\hat{w}_k^\iota]\\
    & = -2 \sum_{k=1}^D \Big\{\frac{cov(x_k,f_i)cov(x_k,\psi_\iota)}{\sigma^2_{x_k\lvert x^{-k}}} - cov(x_k,f_i)\sum_{\substack{h=1 \\ h\neq k}}^{D}\Big( \frac{\rho_{x_h,x_k\lvert x^{-h,k}}}{\sigma_{x_h\lvert x^{-h}}\cdot \sigma_{x_k\lvert x^{-k}}} \Big) cov(x_h,\psi_\iota)\Big\} \\
    &= -2cov(f_i,\mathbf{X})cov(\mathbf{X},\mathbf{X})^{-1}cov(\mathbf{X},\psi_\iota) \\
    &= -2 ( cov(f_i,\psi_\iota) -  cov(f_i,\psi_\iota | \mathbf{X}) ),
\end{align*}

which concludes the proof.

\end{proof}

The single-task and multi-task settings with transformed inputs $\phi_1,\dots,\phi_d$ can be proved with the same passages used in this section to prove the results related to the regression with original features. Therefore, the following theorem concludes the overview of results related to biases, \textbf{proving Theorem \ref{thm:bias} of the main paper}.

\begin{theorem}
    In the asymptotic case, estimating the $i$-th task in a single-task setting, with a set of transformed features $\phi_1\dots\phi_d$, the bias of the model is:

\begin{equation*}
    bias_{d}^i = \sigma^2_{f_i} (1-R^2_{d,i}).
\end{equation*}

Finally, combining the aggregation of targets and the transformation of features, we get the expression reported in Theorem \ref{thm:bias} in the main paper:

\begin{equation}
\begin{split}
    bias_d^\iota &= \sigma^2_{f_i} + \sigma^2_{\psi_\iota}R^2_{d,\iota}  -2 ( cov(f_i,\psi_\iota) -  cov(f_i,\psi_\iota | \mathbf{\Phi}) )\\
    & = \sigma^2_{f_i} - \sigma^2_{\psi_\iota}R^2_{d,\iota}  +2 (cov(\psi_\iota,f_i - \psi_\iota | \mathbf{\Phi}) - cov(\psi_\iota,f_i - \psi_\iota)).
\end{split}
\end{equation}

\end{theorem}

\begin{proof}
    In the first case, with the same considerations of the previous subsection, we get:

\begin{align*}
    bias_d^i &= var_x(\sum_{k=1}^d\E_{\mathcal{T}}[\hat{w}_k^i]\phi_k) + var_x(f_i(x)) - 2 cov(\sum_{k=1}^d\E_{\mathcal{T}}[\hat{w}_k^i]\phi_k,f_i(x))\\
    &=var_x(\sum_{k=1}^d\E_{\mathcal{T}}[\hat{w}_k^i]\phi_k) + var_x(f_i(x)) - 2 cov(\sum_{k=1}^d\E_{\mathcal{T}}[\hat{w}_k^i]\phi_k,f_i(x))\\
    &= var_x(\sum_{k=1}^d\E_{\mathcal{T}}[\hat{w}_k^i]\phi_k) + \sigma^2_{f_i} -2\sigma^2_{f_i}R^2_{d,i}\\
    &= \sigma^2_{f_i} (1-R^2_{d,i}).
\end{align*}

Similarly, in the second case the theorem follows with the same argumentations of Theorem \label{thm:generalBias}, proving also Theorem \ref{thm:bias} in the main paper.
\end{proof}

\section{Additional results: bias-variance comparisons between models}

\textbf{Single vs. Multi task model}
Firstly, focusing on the comparison between the two settings with the original set of $D$ features and comparing the case with a single-task with a multi-task model that learns on an aggregated target that is the average of a cluster of features $\psi_\iota$, the asymptotic decrease of variance when the target is aggregated is equal to: 
\begin{align*}
    \Delta var_D^{i-\iota} &:= var_D^{i}-var_D^{\iota} \\
    & = \frac{\sigma^2_i - \bar{\sigma}_\iota^2}{(n-1)}\cdot\Big\{ \sum_{k=1}^D\frac{\sigma^2_{x_k}}{\sigma^2_{x_k\lvert x^{-k}}} - 2 \sum_{k=1}^{D-1}\sum_{j=k+1}^D \frac{cov(x_k,x_j)\cdot \rho_{x_k,x_j\lvert x^{-k,j}}}{\sigma_{x_k\lvert x^{-k}}\cdot \sigma_{x_j\lvert x^{-j}}} \Big\} \\ 
    &= \frac{\sigma^2_i - \bar{\sigma}_\iota^2}{(n-1)} \cdot \mathrm{D} \overset{\sigma^2_i=\sigma^2_k \forall k\in \mathcal{P}_\iota}{=} (\frac{K-1}{K}) \cdot \frac{\sigma^2_i(1-\bar{\rho}_i)}{(n-1)} \mathrm{D},
\end{align*}

Where $\bar{\rho}_i:=\frac{1}{\frac{K(K-1)}{2}}\cdot \sum_{h\neq k}\rho_{h,k}$ is the average correlation among noises. As discussed in the main paper, remembering that by non-negativity of variance $\bar{\rho}_i\geq \frac{1}{1-K}$, no variance gain corresponds to $\bar{\rho}_i=1$, while the maximum one corresponds to the minimum average correlation $\bar{\rho}_i=\frac{1}{1-K}$. The gain is satisfactory when the noises are independent, $\bar{\rho}_i=0$, leading to $\frac{K-1}{K} \cdot \frac{\sigma^2_i}{(n-1)}\mathrm{D}$. This is in line with the intuition that the learning problem is simpler when the target is an average of less correlated tasks. However, this is balanced by the asymptotic difference of biases, which increases if we aggregate tasks that are less:
\begin{align*}
    \Delta bias_D^{\iota-i} &= \sigma^2_{f_i} - \sigma^2_{\psi_\iota}R^2_{D,\iota}  -2 (cov(\psi_\iota,f_i - \psi_\iota) - cov(\psi_\iota,f_i - \psi_\iota | \mathbf{X})) - \\& \sigma^2_{f_i} + \sigma^2_{f_i} R^2_{D,i}\\
    & = \sigma^2_{f_i} R^2_{D,i} - \sigma^2_{\psi_\iota}R^2_{D,\iota} -2 (cov(\psi_\iota,f_i - \psi_\iota) - cov(\psi_\iota,f_i - \psi_\iota | \mathbf{X}))\\
    & = \sigma^2_{\psi_\iota}R^2_{D,\iota} + \sigma^2_{f_i} R^2_{D,i} -2 ( cov(f_i,\psi_\iota) -  cov(f_i,\psi_\iota | \mathbf{X}) )
\end{align*}

\textbf{Full vs. Aggregated features}
Focusing on the single-task case for simplicity, we can inspect the effect in terms of bias and variance due to the transformation of features from the $D$ original ones to a set of $d<D$ basis functions. The results are valid for general zero-mean transformation of features, although the main interest of our approach is in the aggregation with the mean of $d$ groups of features that form a partition of the original ones. 

The asymptotic decrease of variance in this single-task setting is therefore:

\begin{align*}
    \Delta var_{D-d}^{i} &= \frac{\sigma^2_i}{(n-1)}\cdot\Big\{ \Big[\sum_{k=1}^D\frac{\sigma^2_{x_k}}{\sigma^2_{x_k\lvert x^{-k}}} - \sum_{k=1}^d\frac{\sigma^2_{\phi_k}}{\sigma^2_{\phi_k\lvert \phi^{-k}}}\Big] \\ 
    & - 2\Big[\sum_{k=1}^{D-1}\sum_{j=k+1}^D \frac{cov(x_k,x_j)\cdot \rho_{x_k,x_j\lvert x^{-k,j}}}{\sigma_{x_k\lvert x^{-k}}\cdot \sigma_{x_j\lvert x^{-j}}} \\
    & - \sum_{k=1}^{d-1}\sum_{j=k+1}^d \frac{cov(\phi_k,\phi_j)\cdot \rho_{\phi_k,\phi_j\lvert \phi^{-k,j}}}{\sigma_{\phi_k\lvert \phi^{-k}}\cdot \sigma_{\phi_j\lvert \phi^{-j}}} \Big]\Big\}\\
    &= \frac{\sigma^2_i}{(n-1)}\cdot (D-d).
\end{align*}

On the other hand, the increase of bias is equal to:

\begin{equation*}
    \Delta bias_{d-D}^{i} = \sigma^2_{f_i} (1-R^2_{d,i}) - \sigma^2_{f_i} (1-R^2_{D,i}) = \sigma^2_{f_i} (R^2_{D,i}-R^2_{d,i}).
\end{equation*}

From the variance comparisons, we clearly see that the asymptotic advantage of dimensionality reduction resides in a reduced number of coefficients that need to be estimated. Additionally, the bias shows that it is profitable to reduce features as long as the amount of information shared with the target, measured in terms of the (squared) coefficient of multiple correlation, remains satisfactory.

\section{Algorithm}\label{app:algo}

Algorithm \ref{alg:AggregationLoop} reports the pseudo-code of the auxiliary aggregation loop considered both for the target and the task aggregation phase. In particular, in the first phase, it receives as input the full set of targets, the set of features, the hyperparameter, and an indication of the phase. The loop iterates over targets, forming the partition through the exploitation of the threshold function defined in Algorithm \ref{alg:nonLinMTL}. Similarly, in the second phase, the algorithm receives the current aggregated target and the full set of features, iteratively identifying the feature aggregations exploiting the threshold function for features defined in Algorithm \ref{alg:nonLinMTL}.

\begin{algorithm}[h]
\caption{Auxiliary loop for the iterative aggregation of a set of variables with the mean, used both for targets and features, holds for generic random vector.}\label{alg:AggregationLoop}
\begin{algorithmic}
\Require{Input variables $\bm{z}=\{z_1,\dots,z_{\lvert\bm{z}\rvert}\}$; $D$ input features $\bm{x}=\{x_1,\dots,x_D\}$ if we are aggregating targets, one aggregated target $y=\psi_\iota$ if we are aggregating features for an aggregated task in the second phase, flag \emph{phase} identifies it; tolerance $\epsilon$.}
\Ensure{reduced variables (tasks or features): $\{z_{\mathcal{P}_1},\dots,z_{\mathcal{P}_K}\}$}\\ 
\vspace{-0,3cm}
\Function{\textsc{Aggregation}$(\bm{z},phase, \epsilon,\bm{x},y)$}{}\\
\Comment{Iterative aggregation of sets of variables with the mean, used both for targets and features, holds for generic random vector $\bm{z}$}
            \State{$\bm{\mathcal{P}} \leftarrow \{\}$}
\State{$\mathcal{V} \leftarrow \{\}$} \Comment{Set of already considered elements}
\ForEach {$i \in \{ 1,\dots,\lvert\bm{z}\rvert\}$}
\If{$i\not \in \mathcal{V}$}
    \State $\mathcal{P} \leftarrow \{i\}$
    \State $ \mathcal{V} \leftarrow  \mathcal{V} \cup \{i\}$  
    \ForEach{$j \in \{ i+1,\dots,\lvert\bm{z}\rvert\}$}
    \State $z_\mathcal{P} \leftarrow mean(z_k:k\in\mathcal{P})$ \Comment{Mean of elements in current set}
    \If{$phase==1$}\Comment{First phase, we are aggregating targets}
    \State $threshold \leftarrow$ \textsc{Compute\_threshold\_targets}$( \bm{x},z_\mathcal{P},z_j,\epsilon)$
    \Else 
    \State $\bm{x}_{curr} \leftarrow \{x_k: k\notin \mathcal{V}\} \cup \{mean(x_h:h\in\mathcal{P}_k), k\in \{1,\dots,\lvert\bm{\mathcal{P}}\rvert\}\}$ 
    \State $threshold \leftarrow$ \textsc{Compute\_threshold\_features}$( \bm{x}_{curr},y,z_\mathcal{P},z_j,\epsilon)$
    \EndIf
    \If{$threshold==True$} \Comment{Add $j$-th element to current set}
    \State{$\mathcal{P} \leftarrow \mathcal{P}  \cup \{j\}$}
    \State{$\mathcal{V} \leftarrow  \mathcal{V} \cup \{j\}$  }
    \EndIf
    \EndFor
    \State{$\bm{\mathcal{P}} \leftarrow \bm{\mathcal{P}} \cup \{\mathcal{P}\}$}
\EndIf
\EndFor   
\State{$K \gets \lvert\bm{\mathcal{P}}\rvert$}
\ForEach{$k \in \{ 1,\dots,K\}$}
\State{$z_{\mathcal{P}_k} = mean(z_h:h\in\mathcal{P}_k)$}
\EndFor\\
\Return{$\{z_{\mathcal{P}_1},\dots,z_{\mathcal{P}_K}\}$}
\EndFunction
\end{algorithmic}
\end{algorithm}

\section{Experimental Validation}
\subsection{Synthetic Experiments}\label{app:synth}
This section includes additional results related to subsection \ref{sub:synthExp} of the main paper, where synthetic experiments are introduced. Considering the configuration described in the main paper, where the effect of varying different parameters has been described, Figure \ref{fig:synthMSE2} shows the test performances varying the hyperparameters of the \algnameshortMain approach. From the first column of the figure, we can see that $\epsilon_1$ regulates the propensity to aggregate targets, with a value close to zero that balances the aggregations, large negative value that tends to produce a singleton with all tasks averaged (which becomes detrimental for the final performance on the original tasks when too many aggregations are performed) and large positive values that tend to the single-task case, without any aggregation. In this case, the subsequent feature aggregation phase is beneficial when sufficient information is preserved by the task aggregation. The second column of the figure shows the behavior of the algorithm varying the hyperparameter $\epsilon_2$, which is stable for values close to zero, while it tends to aggregate all features when the value is large, becoming detrimental. 

\begin{figure}
\includegraphics[width=\textwidth]{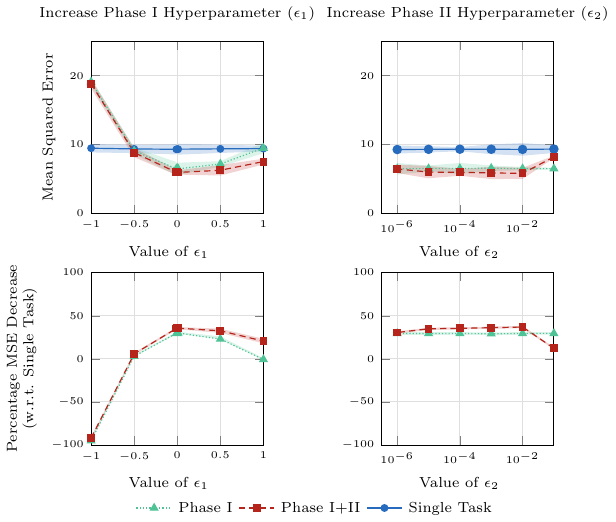}
\caption{Test Mean Squared Error (first row) and corresponding percentage decrease \wrt the single task model (second row), varying one hyperparameter at a time, only aggregating targets (Phase I) or adding the feature aggregation (Phase I+II). The first parameter $\epsilon_1$ refers to the propensity of the user to aggregate tasks. On the other hand, the second parameter $\epsilon_2$ refers to the propensity of the user to aggregate features.} \label{fig:synthMSE2}
\end{figure}

\subsection{Real World Experiments}\label{app:real}
In this subsection we provide additional details and results on the real-world experiments performed. All experiments have been conducted with an HPC-system, specifically a BullSequana XH2000 supercomputer using the 3rd generation of AMD EPYC CPUs (Milan), NVIDIA A100 GPUs, and a 130 Petabyte DDN filesystem. On this machine, we exploited $4$ nodes of CPU and $2$ GPUs for training neural networks.

In the first set of experiments, we consider School and SARCOS data, as described in the main paper, comparing the regression scores obtained with LR, SVR, and MLP after the application of the \algnameshortMain. We selected these datasets to test the algorithm on classical data, exploiting the implementation of some MTL baselines available for benchmarking, as described in Section \ref{sec:experiments}. The SARCOS dataset considers $D$ shared features, therefore the main algorithm has been applied on it. On the contrary, the School dataset presents the same features, measured differently for each school (task). Therefore, we applied the variant of the algorithm discussed in \ref{rem:variants}. Additionally, given the different amounts of samples available for each school, in the aggregated case, we randomly extracted a number of samples from the school with more samples equal to the number of samples of the other school, ordering both samples depending on the target and aggregating consequently. The idea is that in this way, we are aggregating a sample of the best student of one school with a sample of the best student of the other school, then the second best student of the first is aggregated with the second best of the other school, and so on. We considered the NRMSE as a performance score, being a valid elaboration of the MSE, the metric that we are optimizing, and exploiting the fact that it was already implemented for the benchmark MTL approaches considered.

In the second real-world framework, we consider the QM9 dataset, as discussed in Section \ref{sec:experiments}. This is a recent, challenging graph-based dataset, where thousands of molecule structures are provided as inputs, with their characteristics being the targets. In this case, we could therefore apply Algorithm \ref{alg:nonLinMTL}. We had to adapt the input to be a tabular set of features, therefore we averaged the features of nodes and edges, as done in other literature approaches. To provide a fair comparison, we exploited a state-of-the-art library, LibMTL, as described in the main paper, producing $14$ results with a hard parameter sharing graph neural network, trained considering $14$ different weighting strategies (all described in the paper of the library \cite{Lin2022LibMTLAP} and implemented from recent papers. Additionally, we added some results, in terms of MSE, from a recent paper that considers the QM9 dataset similarly, to show the best performance, to our knowledge, of a GNN architecture on these data. From the results, we conclude that our approach is able to outperform the single-task counterparts and to be competitive with graph neural network architectures with hard parameter sharing. However, it does not achieve the performance of the best GNN architecture to the best of our knowledge. This is reasonable since we are transforming the graph inputs to tabular data, losing the graph information. However, we consider this result an instructive example of a tradeoff between reaching the best possible performance with a completely black box approach, such as a complex graph neural network, and an approach like \algnameshortMain, where the entire pipeline is interpretable, aimed to produce a simple yet effective ML workflow. 

Finally, we conclude the experimental section, providing an example of the application of our method to Earth science data. In particular, we consider about $30000$ hydrological sub-basins of Europe, for which we have a satellite target variable and a set of meteorological features and climate indices at each location. In this case, we show the performance of our approach, \wrt the single-task counterparts (which corresponds to $\epsilon=1$, showing a clear improvement of the MSE when a balanced aggregation is performed and a deterioration of the performances when the algorithm aggregates too much or when it does not aggregate enough. We did not consider additional MTL baselines in this context since its main purpose is to conclude the experimental validation section providing a reliable example of application in line with the motivational example that has been mentioned throughout the paper. Additionally, the main applicative interest is to provide an improvement \wrt single task, preserving the interpretability of the pipeline and improving existing climate indices, as will be further described in a future applicative work.
\newpage
%
%
%

\end{document}